\documentclass{article}

\newif\ifarXiv
\arXivfalse

\ifarXiv
  \usepackage{authblk}
  \usepackage{algorithm}
  \usepackage{algorithmic}
  \usepackage[scale = 0.75]{plex-mono}
  \usepackage[T1]{fontenc}

  \usepackage{twoml}
  \usepackage[british]{babel}
\else
  \usepackage[accepted]{icml2023}
  \usepackage[british]{babel}
\fi

\usepackage{amsmath}
\usepackage{amssymb}
\usepackage{amsthm}
\usepackage{booktabs}
\usepackage{censor}
\usepackage{dsfont}
\usepackage{mathtools}
\usepackage{microtype}
\usepackage{mleftright}
\usepackage{multirow}
\usepackage{paralist}
\usepackage{siunitx}
\usepackage{soul}
\usepackage{subcaption}
\usepackage{thm-restate}
\usepackage{wrapfig}
\usepackage{xspace}
\usepackage{csvsimple}

\usepackage[commandnameprefix=ifneeded, defaultcolor = blue]{changes}

\usepackage{tikz}

\usetikzlibrary{arrows}
\usetikzlibrary{calc}
\usetikzlibrary{chains}
\usetikzlibrary{external}
\usetikzlibrary{patterns}

\usepackage{pgfplots}

\pgfplotsset{
  compat = 1.17
}

\usepgfplotslibrary{colorbrewer}
\usepgfplotslibrary{groupplots}
\usepgfplotslibrary{statistics}

\definecolor{bleu}     {RGB}{ 49,140,231}
\definecolor{cardinal} {RGB}{196, 30, 58}
\definecolor{emerald}  {RGB}{ 80,200,120}
\definecolor{lightgrey}{RGB}{230,230,230}

\usepackage{hyperref}
\usepackage[capitalise]{cleveref}

\hypersetup{
  colorlinks,
  linkcolor = bleu,
  citecolor = bleu,
  urlcolor  = bleu,
}

\Crefname{page}{p.}{p.}

\newcommand{\data}[1]{\textsc{#1}\xspace}

\newcommand{\method}{TARDIS\xspace}
\newcommand{\longmethod}{\ul{T}opological \ul{A}lgorithm for \ul{R}obust \ul{DI}scovery of \ul{S}ingularities\xspace}

\pgfplotsset{
  cloud/.style = {%
    axis lines = none,
    width      = 5cm,
    mark size  = 1pt,
    colormap/Reds,
    scatter,
    scatter/use mapped color = {%
      draw = mapped color,
      fill = mapped color
    },
    view     = {-45}{60},
    z buffer = sort,
  },%
}

\newtheorem{corollary}{Corollary}

\newcommand{\diagram} {\mathcal{D}\xspace}
\newcommand{\naturals}{\mathds{N}\xspace}
\newcommand{\reals}   {\mathds{R}\xspace}
\newcommand{\X}       {\mathds{X}\xspace}

\DeclareMathOperator{\Euclidicity}{\mathfrak{E}}
\DeclareMathOperator{\expectation}{E}
\DeclareMathOperator{\EucB}       {\mathds{E}B}
\DeclareMathOperator{\filt}       {\mathcal{F}}
\DeclareMathOperator{\Dist}       {D}
\DeclareMathOperator{\dist}       {d}
\DeclareMathOperator{\distB}      {\dist_{B}}
\DeclareMathOperator{\distH}      {\dist_{H}}
\DeclareMathOperator{\Hom}        {H}
\DeclareMathOperator{\Lk}         {Lk}
\DeclareMathOperator{\PH}         {PH}
\DeclareMathOperator{\PLH}        {PLH}
\DeclareMathOperator{\St}         {St}
\DeclareMathOperator{\vietoris}   {\mathcal{V}}

\begin{document}

\twocolumn[
  \icmltitle{Topological Singularity Detection at Multiple Scales}
  
  \icmlsetsymbol{equal}{*}

  \begin{icmlauthorlist}
    \icmlauthor{Julius von Rohrscheidt}{HM,TUM}
    \icmlauthor{Bastian Rieck}{HM,TUM}
  \end{icmlauthorlist}

  \icmlaffiliation{HM}{Helmholtz Munich}
  \icmlaffiliation{TUM}{Technical University of Munich}

  \icmlcorrespondingauthor{Bastian Rieck}{bastian.rieck@helmholtz-munich.de}

  \vskip 0.3in
]

\printAffiliationsAndNotice{}

\begin{abstract}
  The manifold hypothesis, which assumes that data lies on or close to an unknown
  manifold of low intrinsic dimension, is a staple of modern machine
  learning research.
  However, recent work has shown that
  real-world data exhibits distinct non-manifold
  structures, i.e.\ \emph{singularities}, that can lead to
  erroneous findings.
  Detecting such singularities is therefore crucial as a precursor to
  interpolation and inference tasks.
  We address this issue by developing a topological framework that
  \begin{inparaenum}[(i)]
    \item quantifies the local intrinsic dimension, and
    \item yields a \emph{Euclidicity} score for assessing the
      `manifoldness' of a point along multiple scales.
  \end{inparaenum}
  %
  %
  Our approach identifies singularities of complex spaces, while also
  capturing singular structures and local geometric complexity in
  image data.
\end{abstract}

\section{Introduction}

The ever-increasing amount and complexity of real-world data necessitate
the development of new methods to extract less complex but still
\emph{meaningful} representations of the underlying data.
While numerous methods for approaching this representation learning
problem exist, they all share a common assumption: the underlying data
is supposed to be close to a manifold with small intrinsic dimension,
i.e.\ while the input data may have a large ambient dimension~$N$, there
is an $n$-dimensional manifold with $n \ll N$ that best describes the
data.
For some data sets, such as natural images, this \emph{manifold
hypothesis} is appropriate~\citep{Carlsson09a}.
However, recent research shows evidence that the manifold hypothesis
does not necessarily hold for complex data sets~\citep{brown2023verifying},
and that manifold learning techniques tend to fail for non-manifold
data~\citep{Rieck15b, scoccola2022fiberwise}.
These failures are often the result of \emph{singularities}, i.e.\
regions of a space that violate the properties of a manifold~(see
\cref{sec:Background} for details).
For example, the `pinched torus,' an object obtained by compressing
a random region of the torus~(specifically, a meridian) to a single point, fails to
satisfy the manifold hypothesis at the `pinch point:' this point, unlike
all other points, does \emph{not} have a neighbourhood
homeomorphic to $\reals^2$~(see \cref{fig:Overview} for an illustration).
Since singularities---in contrast to outliers arising from incorrect
labels, for example---often carry relevant
information~\citep{Jakubowski20}, new tools for detecting and
handling non-manifold regions in spaces are needed.

\begin{figure}[tbp]
  \centering
  \includegraphics{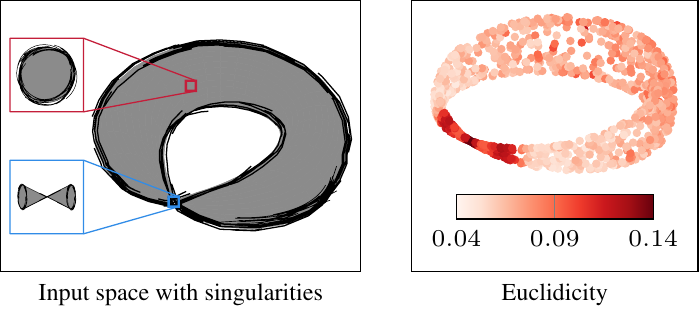}
  \caption{%
    Overview of our method. Given a space with \emph{singularities}, our
    \emph{Euclidicity} score measures the deviation from a space to
    a Euclidean model space.
    Here, \emph{Euclidicity} highlights the singularity at the `pinch point.'
    %
    %
    Please refer to \cref{sec:Methods} for more details.
  }
  \label{fig:Overview}
\end{figure}

\paragraph{Our contributions.}
We develop an \emph{unsupervised representation learning framework for
detecting singular regions in point cloud data}.
Our framework is agnostic with respect to geometric or stochastic
properties of the underlying data and only requires a notion of
intrinsic dimension of neighbourhoods, which, crucially, is allowed to
vary across different points.
Our approach is based on a novel formulation of persistent local
homology~(PLH), a method for assessing the shape of
neighbourhoods at multiple scales of locality.
We use PLH to
\begin{inparaenum}[(i)]
  \item estimate the intrinsic dimension of a point locally, and
  \item define \emph{Euclidicity}, a novel quantity that measures the
    deviation of a point from being Euclidean.
\end{inparaenum}
We also provide theoretical guarantees on the approximation quality for
certain classes of spaces, including manifolds.
\emph{Euclidicity} yields a complementary perspective on data,
highlighting regions where the manifold hypothesis breaks down. We show
the utility of this perspective experimentally on several data sets,
ranging from spaces with known singularities to high-dimensional image
data sets.

\section{Mathematical Background}\label{sec:Background}

We first provide an overview of persistent homology, 
and stratified spaces, as well as their relation to \emph{local
homology}. The former concept constitutes a generic framework for
assessing complex data at multiple scales by measuring topological
characteristics such as `holes' and `voids'~\citep{Edelsbrunner10}, while the latter serves as a general setting to describe singularities, in
which our framework admits advantageous properties.

\paragraph{Persistent homology.}

Persistent homology is a method for computing topological features at
different scales, capturing an intrinsic notion of relevance in terms of
spatial scale parameters.
Given
a finite metric space $(\X,\dist)$, the \emph{Vietoris--Rips complex}
at step~$t$ is defined as the abstract simplicial complex
$\vietoris(\X,t)$, in which an abstract $k$-simplex $(x_0, \dotsc, x_k)$
of points in $\X$ is spanned if and only if $\dist(x_i,x_j) \leq t$ for
all $0 \leq i \leq j \leq k$.\footnote{%
  For readers familiar with persistent homology, we depart from the
  usual convention of using~$\epsilon$ as the threshold parameter since
  we use  it for the scale of our
  persistent local homology calculations.
}
For $t_1 \leq t_2$, the inclusions
$\vietoris(\X,t_1) \hookrightarrow \vietoris(\X,t_2)$ yield
a filtration, i.e.\ a sequence of nested simplicial complexes, which we
denote by $\vietoris(\X,\bullet)$. Applying the $i$th homology functor
to this collection of spaces and inclusions between them induces maps on
the homology level $f_i^{t_1,t_2}\colon \Hom_i(\vietoris(\X,t_1))
\rightarrow \Hom_i(\vietoris(\X,t_2))$ for any $t_1 \leq t_2$. The $i$th
\emph{persistent homology (PH)} of $\X$ with respect to the
Vietoris-Rips construction is defined to be the collection of all these
$i$th homology groups, together with the respective induced maps between
them, and denoted by $\PH_i(\vietoris(\X, \bullet))$. $\PH$ can therefore be
viewed as a tool that keeps track of topological features such as holes and
voids on multiple scales.
For a more comprehensive introduction to $\PH$ in the
context of machine learning, see \citet{Hensel21}.
The so-called `creation' and `destruction' times of these features are
summarised in a \emph{persistence diagram} $\diagram \subset \reals
\times \reals \cup \{\infty\}$, where any point $(b,d) \in \diagram$
corresponds to a homology class that arises at filtration
step~$b$, and lasts until filtration step~$d$. The difference $|d-b|$ is
referred to as the lifetime or eponymous \emph{persistence} of this
homology class.
There are several distance measures for comparing persistence
diagrams, one of them being the \emph{bottleneck distance}, defined as
$\distB(\diagram, \diagram') := \inf_{\gamma} \sup_{x \in \diagram} \| x-\gamma(x) \|_{\infty}$,
where $\gamma$ ranges over all bijections between $\diagram$ and
$\diagram'$.

\paragraph{Stratified spaces.}\label{par:stratified}
%
\begin{figure}
  \centering
  \subcaptionbox{\label{sfig:Figure eight curve}}{
    \begin{tikzpicture}
      \begin{axis}[
        axis lines        = none,
        trig format plots = rad,
        height            = 2.5cm,
        width             = 4.0cm,
        clip              = false,
      ]
        \filldraw[%
          draw = cardinal,
          fill = cardinal!25,
        ] (1, 0) circle (0.20cm);

        \filldraw[draw = cardinal, fill = white] (1, 0) circle (0.10cm);

        \filldraw[%
          draw = bleu,
          fill = bleu!25,
        ] (0, 0) circle (0.20cm);

        \filldraw[draw = bleu, fill = white] (0, 0) circle (0.10cm);

        \addplot[domain=0:2*pi, samples=200] (
          {sin(x)},
          {sin(x) * cos(x)}
        );

        \filldraw (1, 0) circle (1pt);
        \node[anchor = east, xshift = -0.1cm] at (1, 0) {$x$};

        \filldraw (0, 0) circle (1pt);
        \node[anchor = north, yshift = -0.1cm] at (0, 0) {$y$};

      \end{axis}
    \end{tikzpicture}
  }%
  \subcaptionbox{\label{sfig:Regular neighbourhood}}{%
    \begin{tikzpicture}
      \filldraw[draw = cardinal, fill = cardinal!25] (0.5, 0.0) circle (0.50cm);
      \filldraw[draw = cardinal, fill = white      ] (0.5, 0.0) circle (0.25cm);

      \draw              (0.50, -0.50) -- (0.50,  0.50);
      \draw[ultra thick] (0.50,  0.25) -- (0.50,  0.50);
      \draw[ultra thick] (0.50, -0.25) -- (0.50, -0.50);

      \filldraw (0.50, 0.00) circle (1pt);
    \end{tikzpicture}
  }%
  \subcaptionbox{\label{sfig:Singular neighbourhod}}{%
    \begin{tikzpicture}
      \filldraw[draw = bleu, fill = bleu!25] (0.5, 0.0) circle (0.50cm);

      \coordinate (UL) at ($(135:0.50cm) +(0.50, 0.00)$);
      \coordinate (UR) at ($( 45:0.50cm) +(0.50, 0.00)$);
      \coordinate (LR) at ($(315:0.50cm) +(0.50, 0.00)$);
      \coordinate (LL) at ($(225:0.50cm) +(0.50, 0.00)$);

      \draw[ultra thick] (UL) -- (LR);
      \draw[ultra thick] (UR) -- (LL);

      \filldraw[draw = bleu, fill = white  ] (0.5, 0.0) circle (0.25cm);

      \draw (UL) -- (LR);
      \draw (UR) -- (LL);

      \filldraw (0.50, 0.00) circle (1pt);
    \end{tikzpicture}
  }
  \caption{%
    \subref{sfig:Figure eight curve}: Non-manifold space.
    \subref{sfig:Regular neighbourhood}: Annulus around a regular
    point~$x$.
    \subref{sfig:Singular neighbourhod}: Annulus around a singular
    point. The neighbourhood around~$y$ is different from all others.
  }
  \label{fig:Singularities}
\end{figure}
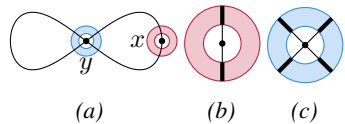

Manifolds are widely studied and particularly well-behaved topological
spaces as they locally resemble Euclidean space near any point. However,
spaces that arise naturally often violate this local homogeneity
condition~(see \cref{fig:Singularities} for an example).
\emph{Stratified spaces} generalise the concept of a manifold to address
\emph{singular spaces}.
Being intrinsically capable of describing a wider class of spaces,
we argue that stratified spaces are the right tool to analyse real-world
data.
In particular, the intrinsic dimension of points is allowed to vary in
this setting, thus leading to high flexibility in comparison to
manifolds.
Subsequently, we define stratified spaces in the setting of
simplicial complexes. A stratified simplicial complex
of dimension $0$ is a finite set of points with the discrete topology.
A stratified simplicial complex of dimension $n$ is an $n$-dimensional
simplicial complex $X$, together with a filtration of closed subcomplexes
$X=X_n \supset X_{n-1} \supset X_{n-2} \supset \dots \supset X_{-1}
= \emptyset$ such that $X_i \setminus X_{i-1}$ is an $i$-dimensional manifold for
all $i$, and such that every point $x \in X$ possesses a distinguished
local neighbourhood 
$U \cong \reals^k \times c^{\circ} L$
in $X$, where $L$ is a compact stratified simplicial complex of
dimension $n-k-1$ and $c^{\circ}$ refers to the open cone
construction~(see \cref{app:Notation}).

\fcolorbox{black}{lightgrey}{%
  \parbox{\dimexpr\linewidth-2\fboxsep-2\fboxrule}{%
    If there exists a neighbourhood $U$ of $x$ that is homeomorphic to
    $\reals^n$, we say that $x$ is a \emph{regular} point, otherwise we
    call $x$ a \emph{singularity}.
  }
}

If $X$ is a manifold, then independently of the point
under consideration, $L$ is given by a sphere since for a manifold,
\emph{any} point is regular by definition. 
By contrast, a small neighbourhood of the pinch point in a pinched
torus can be described as the open cone on the disjoint union of two
circles, and therefore the link $L$ is given by $S^1 \sqcup S^1$ in this
case, whose homology is different from homology of the link of any other
point in the pinched torus~(which is just given by a circle for every
other point). The insets in \cref{fig:Overview}~(left) depict
the different neighbourhoods.

\paragraph{Local homology.}
%
We now formalise the idea of tracking the homology of a link of a point,
following the previous description.
Intuitively, the \emph{local homology} of a point is obtained by the
homology of an infinitesimal small punctured neighbourhood around the
point~(see \cref{app:Proofs} for a more rigorous description).
One can show that the local homology of a singularity usually differs from
the local homology of a regular point. In particular, points that are of
different intrinsic dimensions with respect to the stratified
simplicial complex can be distinguished by local homology; this includes
the disjoint union of manifolds of possibly varying dimensions. This
observation motivates and justifies using local homology for detecting
neighbourhoods of singular points, and serves as the primary motivation
for our novel \emph{Euclidicity} measure in \cref{sec:Euclidicity}.

\section{Related Work}

While manifold learning is concerned with the development of algorithms
that extract geometric information under the assumption that the
given data lie on a manifold, recent work starts to question this
assumption. \citet{brown2023verifying}, for instance, introduce
the \emph{union of manifolds hypothesis}, which augments the manifold
hypothesis to spaces that can be modelled as~(disjoint) unions of
manifolds. Intrinsic dimension is thus allowed to vary across connected
components of such a space. However, singularities are excluded under
this assumption, whereas our method detects the correct intrinsic dimension
for large classes of singular spaces.
We assume a fundamental `singularity-centric' perspective in this paper
and argue that a multi-scale analysis of the local geometry and
topology of data is necessary.
In this context, methods from topological data analysis have started
attracting attention in machine learning~\citep{Hensel21}.
This is particularly due to \emph{persistent homology}, which captures
geometrical and topological properties of the underlying data set on
different scales~\citep{Bubenik20a, Turkes22a}.
The idea of tracking objects on multiple scales can at
least be traced back to~\citep{koenderink1986dynamic,lindeberg1994scale},
with scale space theory playing an eminent role in computer vision.
However, the utility of persistent homology in the context of geometric
singularities in data only came up more recently, since early work in
persistent homology focuses predominantly on the simplification of
functions on manifold domains~\citep{Edelsbrunner02}.
While some research already discusses the utility of persistent homology
for general unsupervised data analysis workflows~\citep{Chazal13a,
Rieck17a, Rieck16a, Rieck15b}, it focuses more on global structures,
whereas singularities are inherently local. A notable exception is
a work by \citet{Wang11a}, which analyses local branching and global
circular features.
We give a brief overview of methods in the emerging field of
topology-driven singularity detection, outlining the differences to our
approach below.

\paragraph{Topology-driven singularity detection.}
Several works assume a local perspective on homology to detect
information about the intrinsic dimensionality of the data or the
presence of certain singularities.
\citet{Rieck20b} use pre-defined
stratifications and \emph{persistent
intersection homology}, a technique developed by \citet{Bendich11},
whereas \citet{fasy2016exploring} and
\citet{bendich2008analyzing} both develop persistent
variants of local homology.
By contrast, \citet{Stolz20a} approximate local
homology as the absolute homology of a small annulus of a given
neighbourhood, resulting in an algorithm for geometric anomaly
detection~(which requires knowing the intrinsic dimension of the data
set).
Moreover, \citet{bendich2007inferring} employ persistence vineyards,
i.e.\ continuous families of  persistence diagrams, to assess the local
homology of a point in a stratified space, whereas
\citet{dey2014dimension} use local homology to estimate the~(global)
intrinsic dimension of hidden, possibly noisy manifolds.

\paragraph{Key differences to existing approaches.}
%
While existing methods overall underscore the relevance of a local
perspective, as well as the use of notions such as \emph{stratified
spaces}, our approach differs from them in essential components.
In comparison to all aforementioned contributions, we capture
additional local geometric information: \ul{we consider multiple
scales of locality in a persistent framework for local homology.}
Concerning the overall construction, \citet{Stolz20a} is the closest to
our method. However, the authors assume that the intrinsic dimension is
known and the proposed algorithm uses one global scale, whereas our
approach
\begin{inparaenum}[(i)]
  \item operates in a multi-scale setting,
  \item provides local estimates of intrinsic dimensionality of the data
    space, and
  \item incorporates model spaces that serve as a comparison.
\end{inparaenum}
We can thus measure the deviation from an idealised manifold,
requiring fewer assumptions on the structure of the input
data~(see \cref{sec:Euclidicity versus single-parameter} for
a comparison).

\section{Methods}\label{sec:Methods}

Our framework \method~(\longmethod) consists of two parts:
\begin{inparaenum}[(i)]
  \item a method to calculate a local intrinsic dimension of the data, and
  \item \emph{Euclidicity}, a measure for assessing the multi-scale
    deviation from a Euclidean space.
\end{inparaenum}
\method is based on the assumption that the intrinsic dimension of data
may not be constant across the data set, and is thus best
described by \emph{local measurements}, i.e.\ measurements in 
a small neighbourhood of a given point.
Since there is no canonical choice for the magnitude of such
a neighbourhood, \method analyses data on multiple scales.
Our main idea involves constructing a collection of local~(punctured)
neighbourhoods for varying locality scales, and calculating
their topological features.
This procedure allows us to approximate local topological
features~(specifically, local homology) of a given point, which we use
to measure the intrinsic dimensionality of a space. Moreover, by
calculating the distance to Euclidean model spaces, we are capable of
detecting singularities in a large range of input data sets.
For the subsequent description of \method, we only assume that data can
be represented as a finite metric space~(i.e.\ as a point cloud).

\subsection{Persistent Intrinsic Dimension}\label{sec:Persistent intrinsic dimension}

For a finite metric space $(\X,\dist)$ and $x \in \X$, let $B_r^{s}(x)
:= \{ y \in \X \mid r \leq \dist(x,y) \leq s \}$ denote the \emph{intrinsic
annulus} of $x$ in $\X$ with respect to radii~$r$ and~$s$.
Moreover, let $\filt$ denote a procedure that takes as input a finite
metric space and outputs an ascending filtration of topological
spaces---such as a Vietoris--Rips filtration.
By applying $\filt$ to the intrinsic annulus of $x$, we obtain
a tri-filtration $(\filt(B_r^{s}(x),t))_{r,s,t}$, where $t$ corresponds
to the respective filtration step that is determined by $\filt$.
Note that this tri-filtration is covariant in $s$ and $t$, but
contravariant in $r$; we denote it by
$\filt(B_{\bullet}^{\bullet}(x),\bullet)$.
Applying $i$th homology to this filtration yields a tri-parameter
persistent module that we call $i$th \textbf{persistent local homology
(PLH)} of $x$, denoted by
$\PLH_i(x; \filt) := \PH_i(\filt(B_{\bullet}^{\bullet}(x),\bullet))$.
\cref{fig:Annulus illustration} illustrates how to obtain an annulus
from a data set and depicts one step of the filtration process.
To the best of our knowledge, this is the first time that PLH is
considered as a multi-parameter persistence module. 
Since the Vietoris--Rips filtration is the pre-eminent filtration in
TDA, we will use $\PLH_i(x) := \PLH_i(x; \vietoris)$ as an
abbreviation.
Before developing ways to detect singularities, we first show that
our PLH formulation enjoys stability properties similar to the
seminal stability theorem in persistent
homology~\citep{Cohen-Steiner07}, making it robust to small parameter
changes.

\begin{restatable}{theorem}{thmstability}
  Given a finite metric space $\X$ and $x \in \X$, let $B_r^{s}(x)$ and
  $B_{r'}^{s'}(x)$ denote two intrinsic annuli with
  $|r-r'| \leq \epsilon_1$ and $|s-s'| \leq \epsilon_2$.
  Furthermore, let $\diagram, \diagram'$ denote the persistence diagrams
  corresponding to $\PH_i(\vietoris(B_{r}^{s}(x),\bullet))$ and
  $\PH_i(\vietoris(B_{r'}^{s'}(x),\bullet))$. Then
  $\frac{1}{2} \distB(\diagram,\diagram') \leq \max \{ \epsilon_1, \epsilon_2 \}$.
  \label{thm:Stability}
\end{restatable}

\begin{figure}[tbp]
  \centering
  \begin{tikzpicture}
    \centering
    \begin{groupplot}[%
      group style = {%
        group size     = 3 by 1,
        horizontal sep = 0.25cm,
      },
      tick style  = {%
        draw = none,
      },
      set layers  = standard,
      axis on top = true,
      ticks       = none,
      xmin        = -0.4,
      xmax        =  0.4,
      ymin        = -0.4,
      ymax        =  0.4,
      height      = 4.75cm,
      width       = 4.75cm,
    ]%
      \nextgroupplot%
        \addplot[only marks, mark size = 1.0pt] table {Data/Torus_2D.txt};

        \coordinate (UR) at (0.35,  0.10);
        \coordinate (LR) at (0.35, -0.10);

        \begin{pgfonlayer}{axis lines}
          \draw[cardinal, thick] (LR) rectangle (0.15, 0.10);
        \end{pgfonlayer}

      \nextgroupplot[
        xmin               =  0.15,
        xmax               =  0.35,
        ymin               = -0.15,
        ymax               =  0.15,
        enlarge x limits   =  0.05,
        unit vector ratio* = 1 1 1,
      ]%

        \addplot[only marks, mark size = 1.0pt] table {Data/Torus_2D_Annulus.txt};%

        \node at (rel axis cs:0.5, 0.925) {$B_{\begingroup\color{bleu}r\endgroup}^{\begingroup\color{cardinal}s\endgroup}(x)$};

        \coordinate[label = {%
            [anchor = north west, label distance = 0.1cm]
          {$x$}
        }] (O) at (0.25, 0.0);

        \filldraw[even odd rule, pattern color = lightgrey, pattern = crosshatch]
            (O) circle (0.05)
            (O) circle (0.10);

        \begin{pgfonlayer}{axis lines}
          \draw[bleu] (O) -- node[label = {%
            [anchor = south, label distance = -0.1cm]
            {$r$}
          }] {} ++(10:0.55cm);

          \draw[cardinal] (O) -- node[label = {%
            [anchor = north west, label distance = -0.15cm]
            {$s$}
          }] {} ++(160:1.05cm);

          \filldraw (O) circle (1.0pt);
        \end{pgfonlayer}

        \draw[bleu]     (O) circle (0.05);
        \draw[cardinal] (O) circle (0.10);

      \nextgroupplot[
        xmin               =  0.15,
        xmax               =  0.35,
        ymin               = -0.15,
        ymax               =  0.15,
        enlarge x limits   =  0.05,
        unit vector ratio* = 1 1 1,
      ]%

        \addplot[only marks, mark size = 1.0pt] table {Data/Torus_2D_Annulus.txt};%

        \coordinate[label = {%
            [anchor = north west, label distance = 0.1cm]
          {$x$}
        }] (O) at (0.25, 0.0);

        \filldraw (O) circle (1.0pt);

        \filldraw[even odd rule, pattern color = lightgrey, pattern = crosshatch]
            (O) circle (0.05)
            (O) circle (0.10);

        \draw[bleu]     (O) circle (0.05);
        \draw[cardinal] (O) circle (0.10);

        \coordinate (P001) at (0.345231,-0.00665766);
\coordinate (P002) at (0.188186,-0.0772709);
\coordinate (P003) at (0.274898,-0.0875508);
\coordinate (P004) at (0.21304,0.0374275);
\coordinate (P005) at (0.166137,-0.034113);
\coordinate (P006) at (0.292153,-0.0898842);
\coordinate (P007) at (0.197114,-0.00571057);
\coordinate (P008) at (0.303857,0.0108488);
\coordinate (P009) at (0.245757,0.0837228);
\coordinate (P010) at (0.155615,0.00490951);
\coordinate (P011) at (0.185478,0.0736299);
\coordinate (P012) at (0.211186,-0.0660789);
\coordinate (P013) at (0.195135,-0.0572786);
\coordinate (P014) at (0.152069,0.00432096);
\coordinate (P015) at (0.20751,-0.084676);
\coordinate (P016) at (0.272308,-0.0937317);
\coordinate (P017) at (0.21948,0.082995);
\coordinate (P018) at (0.21161,-0.0331234);
\coordinate (P019) at (0.236007,0.0500223);
\coordinate (P020) at (0.1848,0.0114699);
\coordinate (P021) at (0.314332,-0.0071042);
\coordinate (P022) at (0.217651,-0.0426942);
\coordinate (P023) at (0.262143,-0.0851891);
\coordinate (P024) at (0.320057,-0.00601757);
\coordinate (P025) at (0.171512,-0.0492367);
\coordinate (P026) at (0.327326,0.0569046);
\coordinate (P027) at (0.336687,-0.0194342);
\coordinate (P028) at (0.210983,0.0471004);
\coordinate (P029) at (0.217524,-0.0661107);
\coordinate (P030) at (0.192301,0.0788238);
\coordinate (P031) at (0.15242,-0.0198346);
\coordinate (P032) at (0.336223,0.00663752);
\coordinate (P033) at (0.293994,-0.0516919);
\coordinate (P034) at (0.190139,0.0406642);
\coordinate (P035) at (0.23454,-0.0741897);
\coordinate (P036) at (0.294462,-0.0765032);
\coordinate (P037) at (0.31472,-0.0740029);
\coordinate (P038) at (0.29984,0.0776311);
\coordinate (P039) at (0.281212,-0.0912245);
\coordinate (P040) at (0.153897,0.0194355);
\coordinate (P041) at (0.159054,0.00931338);
\coordinate (P042) at (0.323245,0.0269332);
\coordinate (P043) at (0.32607,0.0172505);
\coordinate (P044) at (0.163798,-0.0453029);
\coordinate (P045) at (0.256676,0.0543029);
\coordinate (P046) at (0.175665,-0.00999566);
\coordinate (P047) at (0.327044,-0.05046);
\coordinate (P048) at (0.242518,0.095824);
\coordinate (P049) at (0.166036,-0.0332803);
\coordinate (P050) at (0.334628,-0.00698864);
\coordinate (P051) at (0.202399,-0.075093);
\coordinate (P052) at (0.2835,0.0437453);
\coordinate (P053) at (0.182721,0.0115832);
\coordinate (P054) at (0.303006,0.0782397);
\coordinate (P055) at (0.155959,0.0067218);
\coordinate (P056) at (0.301821,-0.0621158);
\coordinate (P057) at (0.332855,-0.0147544);
\coordinate (P058) at (0.254477,0.0646313);
\coordinate (P059) at (0.250312,-0.087451);
\coordinate (P060) at (0.348743,0.00996656);
\coordinate (P061) at (0.318264,-0.0461298);
\coordinate (P062) at (0.195072,0.0688528);
\coordinate (P063) at (0.165252,0.0217276);
\coordinate (P064) at (0.31716,-0.0672742);
\coordinate (P065) at (0.296114,0.0814445);
\coordinate (P066) at (0.300068,0.0749116);
\coordinate (P067) at (0.19838,-0.0491572);
\coordinate (P068) at (0.201423,0.0188209);
\draw (P001) -- (P024);
\draw (P001) -- (P027);
\draw (P001) -- (P032);
\draw (P001) -- (P050);
\draw (P001) -- (P057);
\draw (P001) -- (P060);
\draw (P002) -- (P012);
\draw (P002) -- (P013);
\draw (P002) -- (P015);
\draw (P002) -- (P051);
\draw (P002) -- (P067);
\draw (P003) -- (P006);
\draw (P003) -- (P016);
\draw (P003) -- (P023);
\draw (P003) -- (P036);
\draw (P003) -- (P039);
\draw (P003) -- (P059);
\draw (P004) -- (P019);
\draw (P004) -- (P028);
\draw (P004) -- (P034);
\draw (P004) -- (P068);
\draw (P005) -- (P025);
\draw (P005) -- (P031);
\draw (P005) -- (P044);
\draw (P005) -- (P046);
\draw (P005) -- (P049);
\draw (P006) -- (P016);
\draw (P006) -- (P036);
\draw (P006) -- (P037);
\draw (P006) -- (P039);
\draw (P006) -- (P056);
\draw (P007) -- (P020);
\draw (P007) -- (P046);
\draw (P007) -- (P053);
\draw (P007) -- (P068);
\draw (P008) -- (P021);
\draw (P008) -- (P024);
\draw (P008) -- (P042);
\draw (P008) -- (P043);
\draw (P009) -- (P017);
\draw (P009) -- (P048);
\draw (P009) -- (P058);
\draw (P010) -- (P014);
\draw (P010) -- (P020);
\draw (P010) -- (P031);
\draw (P010) -- (P040);
\draw (P010) -- (P041);
\draw (P010) -- (P046);
\draw (P010) -- (P053);
\draw (P010) -- (P055);
\draw (P010) -- (P063);
\draw (P011) -- (P030);
\draw (P011) -- (P062);
\draw (P012) -- (P013);
\draw (P012) -- (P015);
\draw (P012) -- (P022);
\draw (P012) -- (P029);
\draw (P012) -- (P035);
\draw (P012) -- (P051);
\draw (P012) -- (P067);
\draw (P013) -- (P018);
\draw (P013) -- (P022);
\draw (P013) -- (P025);
\draw (P013) -- (P029);
\draw (P013) -- (P051);
\draw (P013) -- (P067);
\draw (P014) -- (P031);
\draw (P014) -- (P040);
\draw (P014) -- (P041);
\draw (P014) -- (P046);
\draw (P014) -- (P055);
\draw (P014) -- (P063);
\draw (P015) -- (P029);
\draw (P015) -- (P035);
\draw (P015) -- (P051);
\draw (P016) -- (P023);
\draw (P016) -- (P036);
\draw (P016) -- (P039);
\draw (P016) -- (P059);
\draw (P017) -- (P030);
\draw (P017) -- (P048);
\draw (P017) -- (P062);
\draw (P018) -- (P022);
\draw (P018) -- (P067);
\draw (P019) -- (P028);
\draw (P019) -- (P045);
\draw (P019) -- (P058);
\draw (P020) -- (P034);
\draw (P020) -- (P041);
\draw (P020) -- (P046);
\draw (P020) -- (P053);
\draw (P020) -- (P055);
\draw (P020) -- (P063);
\draw (P020) -- (P068);
\draw (P021) -- (P024);
\draw (P021) -- (P027);
\draw (P021) -- (P032);
\draw (P021) -- (P043);
\draw (P021) -- (P050);
\draw (P021) -- (P057);
\draw (P022) -- (P029);
\draw (P022) -- (P067);
\draw (P023) -- (P035);
\draw (P023) -- (P039);
\draw (P023) -- (P059);
\draw (P024) -- (P027);
\draw (P024) -- (P032);
\draw (P024) -- (P043);
\draw (P024) -- (P050);
\draw (P024) -- (P057);
\draw (P025) -- (P044);
\draw (P025) -- (P049);
\draw (P025) -- (P067);
\draw (P027) -- (P032);
\draw (P027) -- (P050);
\draw (P027) -- (P057);
\draw (P028) -- (P034);
\draw (P028) -- (P062);
\draw (P028) -- (P068);
\draw (P029) -- (P035);
\draw (P029) -- (P051);
\draw (P029) -- (P067);
\draw (P030) -- (P062);
\draw (P031) -- (P041);
\draw (P031) -- (P044);
\draw (P031) -- (P046);
\draw (P031) -- (P049);
\draw (P031) -- (P055);
\draw (P032) -- (P042);
\draw (P032) -- (P043);
\draw (P032) -- (P050);
\draw (P032) -- (P057);
\draw (P032) -- (P060);
\draw (P033) -- (P036);
\draw (P033) -- (P056);
\draw (P033) -- (P061);
\draw (P033) -- (P064);
\draw (P034) -- (P062);
\draw (P034) -- (P068);
\draw (P035) -- (P059);
\draw (P036) -- (P037);
\draw (P036) -- (P039);
\draw (P036) -- (P056);
\draw (P036) -- (P064);
\draw (P037) -- (P047);
\draw (P037) -- (P056);
\draw (P037) -- (P061);
\draw (P037) -- (P064);
\draw (P038) -- (P054);
\draw (P038) -- (P065);
\draw (P038) -- (P066);
\draw (P040) -- (P041);
\draw (P040) -- (P053);
\draw (P040) -- (P055);
\draw (P040) -- (P063);
\draw (P041) -- (P046);
\draw (P041) -- (P053);
\draw (P041) -- (P055);
\draw (P041) -- (P063);
\draw (P042) -- (P043);
\draw (P043) -- (P050);
\draw (P043) -- (P060);
\draw (P044) -- (P049);
\draw (P045) -- (P052);
\draw (P045) -- (P058);
\draw (P046) -- (P049);
\draw (P046) -- (P053);
\draw (P046) -- (P055);
\draw (P047) -- (P056);
\draw (P047) -- (P061);
\draw (P047) -- (P064);
\draw (P050) -- (P057);
\draw (P050) -- (P060);
\draw (P051) -- (P067);
\draw (P053) -- (P055);
\draw (P053) -- (P063);
\draw (P053) -- (P068);
\draw (P054) -- (P065);
\draw (P054) -- (P066);
\draw (P055) -- (P063);
\draw (P056) -- (P061);
\draw (P056) -- (P064);
\draw (P057) -- (P060);
\draw (P061) -- (P064);
\draw (P065) -- (P066);

        \foreach \s in {P001,P002,P003,P004,P005,P006,P007,P008,P009,P010,P011,P012,P013,P014,P015,P016,P017,P018,P019,P020,P021,P022,P023,P024,P025,P026,P027,P028,P029,P030,P031,P032,P033,P034,P035,P036,P037,P038,P039,P040,P041,P042,P043,P044,P045,P046,P047,P048,P049,P050,P051,P052,P053,P054,P055,P056,P057,P058,P059,P060,P061,P062,P063,P064,P065,P066,P067,P068}
        {%
          \edef\temp{\noexpand\filldraw[emerald, fill opacity = 0.5]
            (\s) circle (0.015);}
          \temp
        }

        \node at (rel axis cs:0.5, 0.925) {$\vietoris(B_{\begingroup\color{bleu}r\endgroup}^{\begingroup\color{cardinal}s\endgroup}(x), \begingroup\color{emerald}t\endgroup)$};

    \end{groupplot}%

    \draw[thick, cardinal] (UR) -- (group c2r1.north west);
    \draw[thick, cardinal] (LR) -- (group c2r1.south west);

  \end{tikzpicture}

  \caption{%
    The intrinsic annulus $B_r^{s}(x)$ around a point $x$ in
    a metric space $(\X, \dist)$, as well as one filtration step for
    some choice of $t$. By adjusting~$r$ and $s$, we obtain
    a tri-filtration.
  }
  \label{fig:Annulus illustration}
\end{figure}
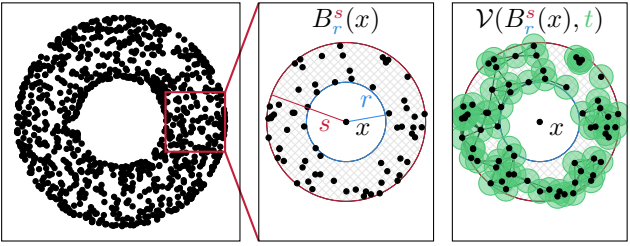

For a finite set of points $\X \subset \reals^N$,
we define the \textbf{persistent intrinsic dimension~(PID)} of $x \in \X$
at scale $\epsilon$ as
$i_x(\epsilon) := \max \mleft\{ i \in \naturals \mid
\ \exists\, r < s <\epsilon \text{ s.t. } \PH_{i-1}(\filt(B_r^s(x),\bullet)) \neq 0\mright\}$.
This measure characterises the intrinsic dimension of data in
a multi-scale fashion. We can also prove that we recover the correct
dimension in case our data set constitutes a manifold sample.
\begin{restatable}{theorem}{thmdimension}
  Let $M \subset \reals^N$ be an $n$-dimensional compact smooth manifold and let $\X := \{ x_1, \dots ,x_S\}$ be a collection of uniform samples from $M$.
  For a sufficiently large $S$ and $\filt = \vietoris$, there exist constants $\epsilon_1, \epsilon_2>0$ such that
  $i_x(\epsilon) = n$
  for all $\epsilon_1 < \epsilon < \epsilon_2$ and any point $x \in \X$.
  Moreover, $\epsilon_1$ can be chosen arbitrarily small by increasing $S$.
  \label{thm:Intrinsic dimension manifolds}
\end{restatable}
\cref{thm:Intrinsic dimension manifolds} implies 
that $i_x(\epsilon)$ computes the correct intrinsic dimension of $M$ in
a certain range of values $\epsilon>0$, provided the sample size is
sufficiently large.
Moreover, $i_x(\epsilon)$ persists in this range, which
suggests considering a collection of $i_x(\epsilon)$ for varying
$\epsilon$ to analyse the intrinsic dimension of $x$.
We also have the following corollary, which specifically addresses
stratified spaces such as the `pinched torus,' implying that our
method can correctly detect the intrinsic dimension of individual
strata. PID is thus capable of handling large classes of `non-manifold'
data sets.
\begin{corollary}
Let $X=X_n \supset X_{n-1} \supset X_{n-2} \supset \dots \supset X_{-1}
= \emptyset$ be an $n$-dimensional compact stratified simplicial
complex, s.t.\ $X_i \setminus X_{i-1}$ is smooth for every $i$. For a fixed $i$,
let $\X_i := \{ x_1, \dots ,x_S\}$ be a collection of uniform samples from
$X_i \setminus X_{i-1}$. For a sufficiently large $S$ and $\filt = \vietoris$, there are constants
$\epsilon_1, \epsilon_2>0$ such that
$i_x(\epsilon) = i$
for all $\epsilon_1 < \epsilon < \epsilon_2$ and any point $x \in \X_i$.
Moreover, $\epsilon_1$ can be chosen arbitrarily small by increasing $S$.
\end{corollary}

\subsection{Euclidicity}\label{sec:Euclidicity}

Knowledge about the intrinsic dimension of a neighbourhood is crucial
for measuring to what extent such a neighbourhood deviates from being
Euclidean.
We refer to this deviation as \emph{Euclidicity}, with the understanding
that low values indicate Euclidean neighbourhoods while high values
indicate singular regions of a data set.
\emph{Euclidicity} can be calculated without stringent assumptions on
manifoldness, requiring only an estimate of the intrinsic dimension~$n$
of~$x$. The previously-described PID estimation procedure is applicable
in this setting and may be used to obtain $n$, for example by
calculating statistics on the set of $i_x(\epsilon)$ for varying
locality parameters $\epsilon$. \emph{Euclidicity} can also
use other dimension estimation procedures, which is
advantageous when additional knowledge about the expected structures
is available~(see \citet{Camastra16a} for a survey).

\fcolorbox{black}{lightgrey}{%
  \parbox{\dimexpr\linewidth-2\fboxsep-2\fboxrule}{%
    The main idea of \emph{Euclidicity} involves assessing how far
    a given neighbourhood of a point~$x$ is from being Euclidean. 
  }
}

To this end, we compare it to a Euclidean model space, measuring the
deviation of their corresponding persistent local homology features.
We first define the Euclidean annulus~$\EucB_r^s(x)$ of $x$ for
parameters $r$ and $s$ to be a set of random uniform samples of $\{
y \in \reals^{n} \mid r \leq \dist(x, y) \leq s \}$ such that $\vert
\EucB_r^s(x) \vert = \vert B_r^s(x) \vert$.
Here, $r$ and~$s$ correspond to the inner and outer radius of the
annulus, respectively.
For $r' \leq r$ and $s \leq s'$ we extend $\EucB_r^s(x)$ by sampling
additional points to obtain $\EucB_{r'}^{s'}(x)$ with $\vert
\EucB_{r'}^{s'}(x) \vert = \vert B_{r'}^{s'}(x) \vert$.
Iterating this procedure leads to a tri-filtration
$\mleft(\filt\mleft(\EucB_r^{s}\mleft(x\mright),t\mright)\mright)_{r,s,t}$
for any filtration $\filt$, following our description in
\cref{sec:Persistent intrinsic dimension}.
We now define the persistent local homology of a Euclidean model space
as
\begin{equation}
  \PLH_i^{\mathds{E}}(x; \filt) := \PH_i\mleft(\filt\mleft(\EucB_{\bullet}^{\bullet}(x),\bullet\mright)\mright).
\end{equation}
Again, for a Vietoris--Rips filtration $\vietoris$, we use a short-form
notation, i.e.\ $\PLH_i^{\mathds{E}}(x) := \PLH_i^{\mathds{E}}(x;
\vietoris)$.
Notice that $\PLH_i^{\mathds{E}}(x)$ implicitly depends on the choice of
intrinsic dimension $n$, and on the samples that are generated randomly.
To remove the dependency on the samples, we consider
$\PLH_i^{\mathds{E}}(x)$ to be a sample of a random variable
$\mathop{\boldsymbol{\PLH}}_i^{\mathds{E}}(x)$. Let $\Dist(\cdot,
\cdot)$ be a distance measure for 3-parameter persistence modules, such
as the \emph{interleaving distance}~\citep{Lesnick15a}.
We then define the \textbf{Euclidicity} of~$x$, denoted by
$\Euclidicity(x)$, as the expected value of these distances, i.e.\
\begin{equation}
  \Euclidicity(x) := \expectation\mleft[\Dist\mleft(\PLH_{n-1}(x),
  \boldsymbol{\PLH}_{n-1}^{\mathds{E}}(x) \mright)\mright].
  \label{eq:Euclidicity}
\end{equation}
This quantity essentially assesses how far~$x$ is from admitting
a regular Euclidean neighbourhood. The choice of $\filt$ and
$\Dist(\cdot, \cdot)$ leaves us multiple ways of implementing
\cref{eq:Euclidicity} in practice. In the following, we describe one
particular implementation with beneficial robustness properties.

\paragraph{Implementation.}
%
Calculating $\Euclidicity(x)$ requires different choices, namely
\begin{inparaenum}[(i)]
  \item a range of locality scales,
  \item a filtration, and
  \item a distance metric between filtrations~$\Dist$.
\end{inparaenum}
Using a grid~$\Gamma$ of possible radii~$(r, s)$ with $r < s$, we
approximate \cref{eq:Euclidicity} using the \emph{mean of the bottleneck
distances of fibred Vietoris--Rips barcodes}, i.e.\
\begin{equation}
    \Euclidicity(x) \approx
  \Dist\mleft(\PLH_{i}\mleft(x\mright),\PLH_{i}^{\mathds{E}}(x)\mright)
  := \frac{1}{C} \!\!\! \sum_{(r, s) \in \Gamma} \!\!\! \distB^{r, s}
  \label{eq:Euclidicity implementation}
\end{equation}
where $C$ is equal to the cardinality of the grid, and
%
$
\distB^{r, s} \!\!\! := \distB\mleft(
    \PH_i\mleft(\vietoris\mleft(B_{r}^{s}\mleft(x\mright), \bullet\mright)\mright),
    \PH_i\mleft(\vietoris\mleft(\EucB_{r}^{s}\mleft(x\mright), \bullet\mright)\mright)
  \mright)$,
with $\PLH_{i}^{\mathds{E}}(x)$ referring to a sample from a Euclidean annulus
of the same size as the intrinsic annulus around~$x$.
\cref{eq:Euclidicity implementation} can be implemented using effective
persistent homology calculation methods~\citep{Bauer21a}, thus
permitting an integration into existing TDA and machine learning frameworks~\citep{gudhi:urm, Tauzin20a}.
\cref{app:Pseudocode} provides pseudocode implementations, while
\cref{sec:Experiments} discusses how to pick these parameters in
practice. We make our
framework publicly available.\footnote{%
  \ifarXiv%
  See \url{https://github.com/aidos-lab/PLH} for the code and experiments.
  \else%
  See the supplementary materials for the code and experiments.
  \fi
}

As the following theorem shows, our approximation of
\cref{eq:Euclidicity implementation} is justified in the sense that for
smooth manifolds, $\Euclidicity(x)$ tends to be arbitrarily small in
a large-sample regime.
\begin{restatable}{theorem}{thmmanifolds}
  Let $M \subset \reals^N$ be a smooth $n$-dimensional manifold and
  let $\X \subset M$ be a finite sample of size $S:= | \X |$. For a given
  $\epsilon > 0$, sufficiently large $S$ and a point $x \in \X$, there
  exists  $s_{\epsilon} > 0$ that only depends on $\epsilon$ and the
  curvature of x (w.r.t.\ $M$), such that the approximation of $\Euclidicity(x)$
  via \cref{eq:Euclidicity implementation} is bounded above by
  $\epsilon$, for any grid $\Gamma$ with maximum outer radius
  $s_{\epsilon}$.
  \label{thm:Manifold consistency}
\end{restatable}

\paragraph{Properties.}
%
The main appeal of our formulation is that calculating both PID and
Euclidicity does not require strong assumptions about the input data: we
only assume that the intrinsic dimension~$n$ of the data is
significantly lower than the ambient dimension~$N$. Treating dimension
as a local quantity that may vary across multiple scales makes
Euclidicity broadly applicable. Moreover, as we showed in
\cref{sec:Persistent intrinsic dimension}, our method is
\emph{guaranteed} to yield the right values for manifolds and stratified
simplicial complexes. This increases both the
practical applicability and expressivity, enabling our framework to
handle unions of manifolds of varying dimensions, for instance.
Euclidicity thus generalises to a larger class of spaces than existing
approaches~\citep{brown2023verifying}, permitting a more fine-grained
structural assessment.

\paragraph{Limitations.}
%
Our implementation of Euclidicity makes use of the Vietoris--Rips
complex, which is known to grow exponentially with increasing
dimensionality.
While all calculations of \cref{eq:Euclidicity} can be performed
\emph{in parallel}---thus substantially improving scalability vis-à-vis
persistent homology on the complete input data set, both in terms of
dimensions and in terms of samples---the memory requirements for
a full Vietoris--Rips complex construction may still prevent our
method to be applicable for some high-dimensional data sets.
This can be mitigated by using a different
filtration~\citep{Anai20a, Sheehy13a}. Our proofs do \emph{not}
assume a specific filtration, and we leave the derivation of
filtration-specific theoretical properties for future work. 
Finally, we remark that the reliability of the Euclidicity score depends on
the validity of the intrinsic dimension; otherwise, the comparison does
not take place with respect to the appropriate model space.

\section{Experiments}\label{sec:Experiments}

We demonstrate the expressivity of \method in different settings,
showing that it
\begin{inparaenum}[(i)]
  \item calculates the correct intrinsic dimension, and
  \item detects singularities
\end{inparaenum}
when analysing data sets with known singularities.
We also conduct a brief comparison with one-parameter approaches,
showcasing how our multi-scale approach results in more stable outcomes.
Finally, we analyse Euclidicity scores of benchmark and real-world
datasets, giving evidence that our technique can be used as a measure
for the geometric complexity of data.

\subsection{Parameter Selection}\label{sec:Parameter selection}

Since \cref{eq:Euclidicity} intrinsically incorporates multiple scales
of locality, we need to specify an upper bound for the radii~($r_{\min},
r_{\max}, s_{\min}, s_{\max}$) that define the respective annuli in practice.
Given a point $x$, we found the following procedure to be useful in
practice: we set $s_{\max}$, i.e.\
the maximum of the outer radius, to the distance to the $k$th
nearest neighbour of a point, and~$r_{\min}$, i.e\ the minimum inner
radius, to the smallest non-zero distance to a neighbour of~$x$.
Finally, we set the minimum outer radius~$s_{\min}$ and the maximum
inner radius~$r_{\max}$ to the distance to the $\lfloor \frac{k}{3}
\rfloor$th nearest neighbour.
While we find $k = 50$ to yield sufficient results, spaces with a high
intrinsic dimension may require larger values.
The advantage of using such a parameter selection procedure is that
it works in a data-driven manner, accounting for differences in
density.
Since our approach is inherently \emph{local}, we need to find a balance
between sample sizes that are sufficiently large to contain topological
information, while at the same time being sufficiently small to retain
a local perspective. We found the given range to be an appropriate
choice in practice.
As for the number of steps, we discretise the parameter range using~$20$
steps by default. Higher numbers are advisable when there are large
discrepancies between the radii, for instance when $s_{\max} \gg
r_{\max}$.

\subsection{Persistent Intrinsic Dimension is Expressive}

\begin{table}
  \centering
  \sisetup{
    separate-uncertainty = true,
    table-text-alignment = center,
    table-format         = 1.2,
    tight-spacing,
  }
  \setlength{\tabcolsep}{2.0pt}
  \caption{%
    Dimensionality estimates for the concatenation of $S^1$ and $S^2$,
    denoted by $S^1 \vee S^2$. Our PID measure is more capable of
    detecting the changes in dimensionality that arise from the
    concatenation.
  }
  \label{tab:Dimensionality estimates comparison}
  \begin{tabular}{@{}llSS[table-format=1.2(2)]S}
    \toprule
    {}                     & {\textsc{Method}} & {\textsc{min}} & {$\mu \pm \sigma$} & {\textsc{max}}\\
    \midrule
    \multirow{4}{*}{1D}    & \texttt{lpca}     & 1.00           & 1.42 \pm 0.78      & 3.00\\
                           & \texttt{twoNN}    & 0.83           & 1.00 \pm 0.07      & 1.20\\
                           & \texttt{DANCo}    & 1.00           & 1.00 \pm 0.01      & 1.16\\
                           &  PID              & 1.00           & 1.12 \pm 0.24      & 1.97\\
    \midrule
    \multirow{4}{*}{2D}    & \texttt{lpca}     & 2.00           & 2.88 \pm 0.32      & 3.00\\
                           & \texttt{twoNN}    & 1.01           & 1.90 \pm 0.36      & 2.53\\
                           & \texttt{DANCo}    & 1.00           & 2.10 \pm 0.32      & 3.00\\
                           &  PID              & 1.52           & 1.95 \pm 0.06      & 2.08\\
    \bottomrule
  \end{tabular}
\end{table}

We first analyse the behaviour of persistent intrinsic dimension~(PID) on
samples from a space obtained by concatenating $S^1$~(a circle) and
$S^2$~(a sphere) at a gluing point.
\cref{tab:Dimensionality estimates comparison} shows a comparison of
PID with state-of-the-art dimensionality estimators.\footnote{%
  Method names are taken from the \texttt{scikit-dimension} toolkit.
  See \cref{app:PID} for more details.

}
We find that PID outperforms all estimators in terms of mean and
standard deviation for the 2D points, thus correctly indicating that the
majority of all points admit non-singular 2D neighbourhoods.
For the 1D points, the mean of the dimensionality estimate of PID is
still close to the ground truth, while its standard deviation and
maximum correctly capture the fact that some 1D points are situated
closer to the gluing point. 
\cref{fig:S1vS2 dimensionality} exemplifies this behaviour and shows
a comparison between one of the estimators and our PID measure.
PID is more nuanced in capturing changes in local intrinsic dimension as
one approaches the gluing point, correctly showing that points of the
sphere admit 2D neighbourhoods.
This behaviour is in line with our philosophy of considering
dimensionality as an inherently local phenomenon. In case such behaviour
is not desirable for a specific data set, Euclidicity calculations
support \emph{any} dimensionality estimator; since such estimators do
not come with strong guarantees such as \cref{thm:Intrinsic dimension
manifolds}, their choice must be ultimately driven by the data set at
hand.  
See \cref{app:PID} for a more detailed analysis of these estimates.

\begin{figure}[tbp]
  \centering
  \subcaptionbox{\texttt{twoNN}}{%
    \includegraphics[width=0.50\linewidth, page = 1]{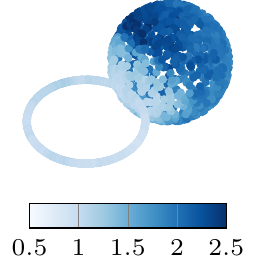}%
  }%
  \subcaptionbox{PID}{%
    \includegraphics[width=0.50\linewidth, page = 2]{Figures/S1vS2_dimensionality}%
  }
  \caption{%
    Dimensionality estimates. PID is more nuanced
    in capturing changes in dimensionality, assigning $\approx 1$ to
    almost all points of the circle, i.e.\ $S^1$, while highlighting
    that points closer to $S^2$ exhibit an increase in dimensionality.
  }
  \label{fig:S1vS2 dimensionality}
\end{figure}

%
In practice, the sample density may not be sufficiently high for
\cref{thm:Intrinsic dimension manifolds} to apply.
This means that there may appear artefact homological features in
dimensions \emph{higher} than the intrinsic dimension of a given space.
We thus recommend to only consider features that exceed a certain persistence
threshold in comparison to the persistence of features of lower
dimension: for any data point~$x$ and the respective intrinsic annulus
$B_r^s(x)$, we suggest to eliminate all topological features whose lifetimes are
smaller than the maximum lifetime of features in one dimension below.
This results in markedly stable estimates of intrinsic dimension, which
are less prone to overestimations.

\subsection{Euclidicity Captures Known Singularities}

Before applying Euclidicity to high-dimensional spaces with unknown
characteristics, we first analyse its behaviour on spaces with known
singularities.
\cref{fig:Overview} shows that Euclidicity is capable of
detecting the singularity of the `pinched torus.' Of particular
relevance is the fact that Euclidicity also highlights that points in
the vicinity of the singular point are \emph{not} fully regular. This is
an important property for practical applications  since it implies that
Euclidicity can detect such \emph{isolated singularities} even in the
presence of sampling errors.

Another prototypical example of singular spaces is given by $S^n \vee
S^n$, the \emph{wedge} of two $n$-dimensional spheres, which is obtained
by gluing them together at one point.
Denoting the gluing point by $x_0$, for a suitable triangulation of
$X = S^n \vee S^n$, this space can be stratified by $X
\supset \{ x_0 \}$.
To assess the utility of \method, we apply it to samples of such wedged
spheres of dimensions $\{2,3,4\}$, calculating their respective
Euclidicity scores.
Since larger intrinsic dimensions require higher sample sizes to
maintain the same density, we start with a sample size of $20000$ in
dimension $2$ and increase it consecutively by a factor of $10$.
We then calculate Euclidicity of $50$~random points in the respective
data set, and additionally for the singular point $x_0$.

\cref{sfig:Euclidicity wedged spheres} shows the results of our
experiments.
We observe that the singular point exhibits a \emph{significantly higher
Euclidicity} score than the random samples. Moreover, we find that
Euclidicity scores of non-singular points exhibit a high degree of
variance across the data, which is caused by the fact that the sampled
data does not perfectly fit the underlying space the points are being
sampled from.
This strengthens our main argument: assessing whether a specific point
is Euclidean does not require a binary decision but a continuous measure
such as Euclidicity.

\paragraph{Stability.}
%
\cref{thm:Stability} predicts that Euclidicity estimates are
stable. Given the use of mini-batches in machine learning, we first
assess whether Euclidicity is \emph{robust towards
sampling}: repeating the calculations for the `pinched torus' over
different batches results in highly similar distributions that are not
distinguishable according to Tukey's range test~\citep{Tukey49} at the
$\alpha = 0.05$ confidence level. Moreover, choosing larger locality scales still
enables us to detect singularities at higher computational costs and
incorporating larger parts of the point cloud. Please refer to
\cref{app:Stability} for a more detailed discussion of this aspect.

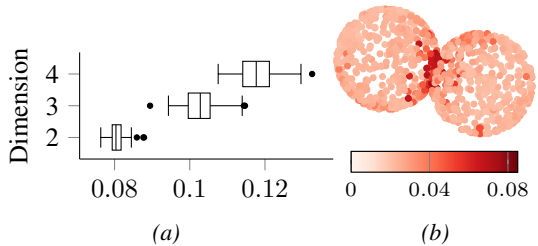
\begin{figure}[tbp]
  \centering
  \subcaptionbox{\label{sfig:Euclidicity wedged spheres}}{%
    \begin{tikzpicture}
      \pgfplotsset{%
        boxplot/draw direction = x,
        %
        %
      }
      \begin{axis}[%
        axis x line*     = bottom,
        axis y line*     = none,
        ytick            = {1, 2, 3},
        yticklabels      = {2, 3, 4},
        height           = 3.0cm,
        width            = 0.60\linewidth,
        mark size        = 1pt,
        tick align       = outside,
        ylabel           = {Dimension},
        xticklabel style = {%
          /pgf/number format/fixed,
          /pgf/number format/precision=2,
        },
      ]
        \addplot[boxplot] table[
          col sep = comma,
          y       = dim_3,
        ] {Data/Wedged_spheres_Euclidicity.csv};

        \addplot[boxplot] table[
          col sep = comma,
          y       = dim_4,
        ] {Data/Wedged_spheres_Euclidicity.csv};

        \addplot[boxplot] table[
          col sep = comma,
          y       = dim_5,
        ] {Data/Wedged_spheres_Euclidicity.csv};
      \end{axis}
    \end{tikzpicture}%
  }%
  \subcaptionbox{\label{sfig:Wedged spheres Euclidicity}}{%
    \includegraphics[width=0.33\linewidth]{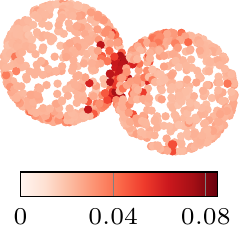}%
  }%
  \caption{%
    \subref{sfig:Euclidicity wedged spheres}:
    Euclidicity scores of wedged spheres for different dimensions. High
    values indicate singular points/neighbourhoods. The Euclidicity of
    the singular point always constitutes a clear positive outlier.
    In 2D, \emph{Euclidicity}~\subref{sfig:Wedged spheres Euclidicity}
    results in a clearly-delineated singular region.
  }
\end{figure}

\paragraph{Comparison to single-parameter approach.}
%
%
We find that our Euclidicity measure also leads to significantly more stable
results than a comparable single-parameter approach for geometry-based
anomaly detection~\citep{Stolz20a}. We analyse this behaviour in more
detail in \cref{sec:Euclidicity versus single-parameter}, observing that 
a single-parameter approach, which examines data with a constant
global scale, results in many `false positives,' i.e.\ points with high
anomaly scores that in fact \emph{do} admit a Euclidean neighbourhood.
This behaviour will only be exacerbated in higher dimensions and sparser
data, highlighting that the use of multiple locality scales by
Euclidicity is advantageous. 

\fcolorbox{black}{lightgrey}{%
  \parbox{\dimexpr\linewidth-2\fboxsep-2\fboxrule}{%
    The preceding experiments verified PID and Euclidicity in terms of
    \emph{expressivity}, i.e.\ detecting singularities in spaces with
    a known ground truth, and \emph{robustness}.
    Hence, our subsequent experiments will proceed with the analysis of
    high-dimensional spaces, exemplifying potential use cases of
    Euclidicity.
  }
}

\subsection{Euclidicity of High-Dimensional Spaces}\label{subsec:MNISTandFashionMNIST}

To test \method in an unsupervised setting, we calculate Euclidicity scores
for the \data{MNIST} and \data{FashionMNIST} data sets, selecting
mini-batches of $1000$ samples from a subsample of $10000$ random images
of these data sets. We `flatten' all greyscale images, thus representing each image
as a high-dimensional point, and a collection of images as a point
cloud, for which we can calculate Euclidicity values. This procedure
follows common practice in dimensionality reduction~\citep{mcinnes2018umap, Moor20c}.
Following \citet{pope2021intrinsic}, we assume an intrinsic dimension of
$10$; moreover, we use $k = 50$ neighbours for local scale estimation.
To ensure that our results are representative, we repeat all
calculations for five different subsamples. Euclidicity scores range
from $[1.1, 5.3]$ for \data{MNIST}, and $[1.3, 5.6]$ for
\data{FashionMNIST}.
The scores of the two datasets appear to be following different
distributions~(see \cref{app:emp_dist} for a visualisation and a 
detailed depiction of the distributions).

\begin{figure}[tbp]
  \centering
  \subcaptionbox{\data{MNIST}\label{sfig:MNIST examples}}{
    \includegraphics[width=0.300\linewidth]{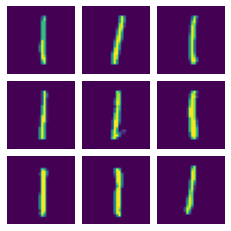}%
    \includegraphics[width=0.300\linewidth]{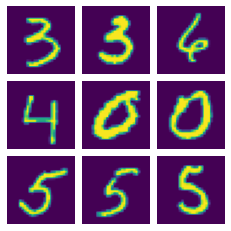}%
    \includegraphics[width=0.300\linewidth]{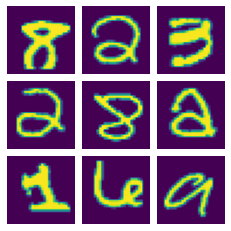}
  }
  \subcaptionbox{\data{FashionMNIST}\label{sfig:FashionMNIST examples}}{
    \includegraphics[width=0.300\linewidth]{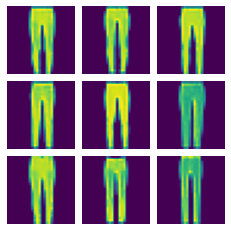}%
    \includegraphics[width=0.300\linewidth]{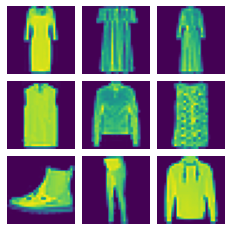}%
    \includegraphics[width=0.300\linewidth]{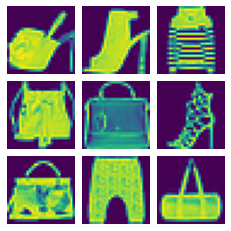}
  }
  \caption{%
    Euclidicity captures local sample complexity in an unsupervised
    manner. From left to right: sample images exhibiting low, median,
    and high Euclidicity, respectively.
  }
  \label{fig:Vision data sets examples}
\end{figure}

\cref{fig:Vision data sets examples} shows a selection of~$9$ images,
corresponding to the lowest, median, and highest Euclidicity scores,
respectively.
We observe that high Euclidicity scores correspond to images with a high
degree of non-linearity, whereas low Euclidicity scores correspond to
images that exhibit less complex structures: for \data{MNIST}, these are
digits of `1.'
Interestingly, we observe the same phenomenon for
\data{FashionMNIST}, where images with low Euclidicity~(`pants') possess
less geometric complexity in contrast to images with high Euclidicity.
Since low Euclidicity can also be seen as an indicator of how close
a neighbourhood is to being \emph{locally linear}, this finding
hints at the existence of simple substructures in the data, prompting
the use of Euclidicity as an unsupervised measure of geometric
complexity.

\begin{figure}[tbp]
  \centering
  \subcaptionbox{Labels}{%
    \includegraphics[width=0.475\linewidth, page = 2]{Figures/MNIST_Euclidicity}%
  }%
  \quad%
  \subcaptionbox{Euclidicity scores}{%
    \includegraphics[width=0.475\linewidth, page = 1]{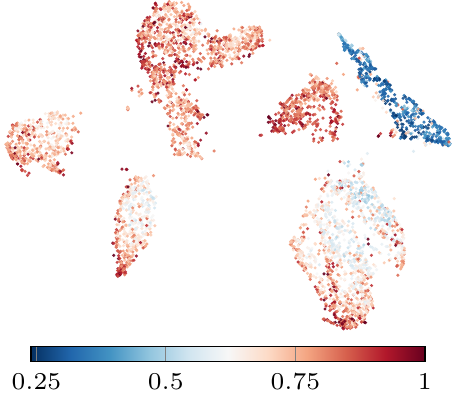}%
  }
  \caption{%
    A UMAP~\citep{mcinnes2018umap} embedding of the \data{MNIST} data,
    showing Euclidicity scores and class labels.
  }
  \label{fig:MNIST embedding}
\end{figure}

To analyse the relationship between geometric complexity and
Euclidicity, we focus on \data{MNIST}.
\cref{fig:MNIST embedding} shows an embedding of the data, with labels
and \emph{normalised Euclidicity scores}~(by the maximum)
highlighted~(Euclidicity was only calculated on the raw data; the
embedding is just used for visualisation purposes).
We find that low Euclidicity scores are prevalent in clusters of `1's,
whereas `5's are assigned both lower than average and higher than average
Euclidicity scores~(similar patterns hold for other classes).
This lends credence to considering \data{MNIST} to consist of
a \emph{union of manifolds}, which are, however, not necessarily split
along the different classes of digits~(meaning that, as our scores
indicate, even images having the same label may not `live' on the same
manifold).

To further assess this hypothesis, we analyse the connection between
a classifier's ability to correctly classify a given sample from the
test set, and its corresponding normalised Euclidicity score. To this
end, we trained a simple neural network classifier for both \data{MNIST} and
\data{FashionMNIST}, and subsequently compared Euclidicity scores of
correctly and incorrectly classified samples. We note that the misclassified
samples admit Euclidicity scores that are significantly higher than
a random equally-sized sample of the correctly classified images.
For \data{MNIST}, the average Euclidicity of correctly classified
samples is $0.33$, whereas misclassified samples admit an average score of
$0.39$. The respective averages for \data{FashionMNIST} are $0.31$ and
$0.35$, respectively.
Welch's t-test indicates significance of these results; please refer to
\cref{app:emp_dist} for details on the model architecture and
statistical significance.
This insight encourages future uses of Euclidicity in models to improve
or explain classification issues, either as an unsupervised
preprocessing step, or by incorporating Euclidicity into the model
architecture itself via appropriate loss terms.

\subsection{Euclidicity of Cytometry Data}\label{sec:Euclidicity iPSC}

To highlight the utility of Euclidicity in unsupervised representation
learning, we also calculate it on an induced pluripotent stem cell~(iPSC)
reprogramming data set~\citep{Zunder15}.
The data set consists of 33 variables and around 220,000 samples; it
depicts a progression of so-called \emph{fibroblasts} diverging, and
splitting into two different lineages.
The data set is known to contain branching structures~\citep{Bhaskar22a}
that can best be extracted using PHATE~\citep{Moon19}, a non-linear
dimensionality reduction algorithm.
\ul{We only employ this algorithm for visualisation purposes; all
Euclidicity calculations are performed on the original data}.
We use \texttt{twoNN} to estimate dimensionality, obtaining a mean
intrinsic dimension of~16, and select parameters as described in
\cref{subsec:MNISTandFashionMNIST}.

\begin{figure}[tbp]
  \centering
  \subcaptionbox{10,000 points\label{sfig:iPSC}}{%
    \includegraphics[width=0.33\linewidth]{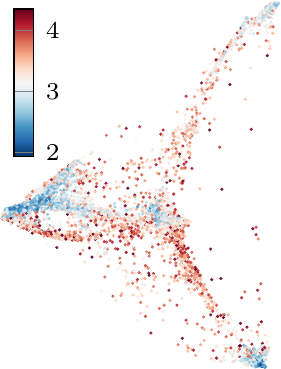}%
  }%
  \subcaptionbox{5,000 points}{%
    \includegraphics[width=0.33\linewidth]{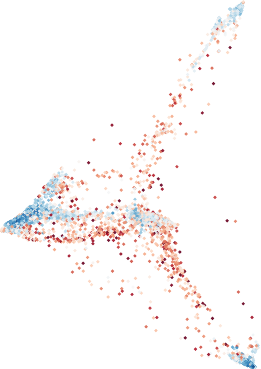}%
  }%
  \subcaptionbox{1,000 points}{%
    \includegraphics[width=0.33\linewidth]{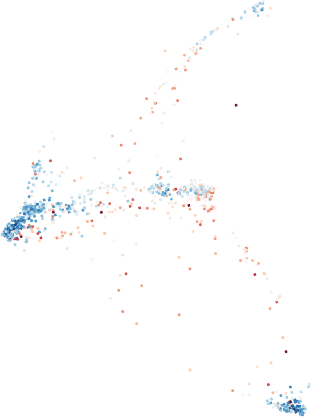}%
  }%
  \caption{%
    PHATE embeddings of a cytometry data set, with colours based on
    Euclidicity. Euclidicity highlights dense non-singular regions and
    remains stable under subsampling the iPSC data set.
    Minor variations in the point cloud shape are due to the PHATE
    embedding algorithm; Euclidicity was always calculated on the raw data.
 }
 \label{fig:iPSC subsampling}
\end{figure}

\cref{sfig:iPSC} shows an embedding obtained via PHATE and
the Euclidicity scores of the original data. We find that high
Euclidicity scores occur in regions that exhibit a lower
density in the PHATE embedding; such points turn out to be situated in
lower-dimensional subspaces.\footnote{%
  However, low-density regions in the PHATE visualisation
  need not necessarily correspond to low-density regions in the original
  dataset.
}
We verify this fact using the \texttt{twoNN} dimensionality estimates
depicted in \cref{fig:iPSC dimensionality}.
More specifically, we calculated the intrinsic dimension for the
subsample, observing that the interquartile range for the 1,000 points
with \emph{highest Euclidicity} is around 12--14, whereas the
interquartile range of the 1,000 \emph{lowest Euclidicity} points ranges
between around 13--16.
Again, we used the \texttt{twoNN} algorithm for intrinsic dimensionality
estimates~(using $k = 50$ nearest neighbours).
Since lower-dimensional points in a space can be
considered \emph{singular} in the sense of stratified spaces, this is
further evidence for Euclidicity to be a useful tool for detecting
non-manifold regions in data. 
Finally, we remark that our Euclidicity estimates remain stable under
\emph{subsampling}.
\cref{fig:iPSC subsampling} depicts subsamples of different sizes
for which we calculated Euclidicity~(on the raw data, respectively,
using PHATE to obtain embeddings). Euclidicity distributions remain
stable and the same phenomena are highlighted for each subsample. 

This analysis constitutes a proof of concept, showcasing how to use
Euclidicity to `probe' data sets and enrich general unsupervised data
analysis workflows.

\begin{figure}[tbp]
  \centering
  \begin{tikzpicture}
    \pgfplotsset{%
      boxplot/draw direction = x,
    }
    \begin{axis}[%
      axis x line*     = bottom,
      axis y line*     = none,
      ytick            = {1, 2},
      yticklabels      = {High, Low},
      height           = 3.0cm,
      width            = \linewidth,
      mark size        = 1pt,
      tick align       = outside,
      xlabel           = {Intrinsic dimension},
      ylabel           = {Euclidicity},
      xticklabel style = {%
        /pgf/number format/fixed,
        /pgf/number format/precision=2,
      },
    ]
      \addplot[boxplot] table[y index = 0] {Data/iPSC_ID_high_Euclidicity.txt};
      \addplot[boxplot] table[y index = 0] {Data/iPSC_ID_low_Euclidicity.txt};
    \end{axis}
  \end{tikzpicture}%
  \caption{%
    A comparison of intrinsic dimension estimates computed for points in
    the iPSC dataset that admit low~(top) and high~(bottom) Euclidicity
    scores.
  }
  \label{fig:iPSC dimensionality}
\end{figure}
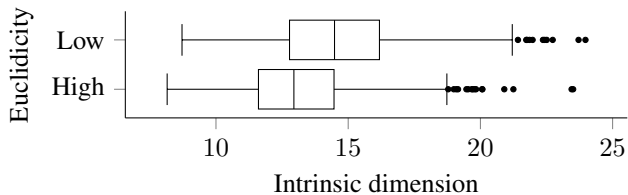

\section{Discussion}

We present \method, a framework for locally estimating
\begin{inparaenum}[(i)]
  \item the intrinsic dimension via PID, the \emph{persistent intrinsic
    dimension}, and
  \item the `manifoldness' via \emph{Euclidicity}, a multi-scale measure of the
    deviation from Euclidean space of point clouds.
\end{inparaenum}
Our method is based on a novel formulation of persistent local homology
as a multi-parameter approach, and we provide theoretical guarantees for
it in a dense sample setting.
Our experiments showed significant improvements of stability compared to
geometry-based anomaly detection methods with fixed locality scales, and
we found that Euclidicity can detect singular regions in
data sets with known singularities. Moreover, using high-dimensional
benchmark and real-world data sets, we also find that Euclidicity
can serve as an unsupervised measure of geometric complexity.

\paragraph{Future work.}
%
For future work, we envision two relevant research directions.
First and foremost will be the inclusion of Euclidicity into
machine learning models to make them `singularity-aware.'
In light of our experiments in \cref{subsec:MNISTandFashionMNIST} and
\cref{sec:Euclidicity iPSC}, we believe that Euclidicity will be
particularly useful in unsupervised scenarios, or provide an additional
weight in classification settings~(to ensure that singular examples are
being given lower confidence scores). Moreover, Euclidicity could be
used in the detection of adversarial samples---a task for which
knowledge about the underlying topology of a space is known to be
crucial~\citep{Jang20a}.
As a second direction, we want to further improve the properties of
Euclidicity itself. To this end, we plan to investigate if incorporating
custom distance measures for three-parameter persistence modules, i.e.\
different metrics for \cref{eq:Euclidicity implementation}, lead to
improved results in terms of stability, expressivity, and computational
efficiency.
Moreover, we hypothesise that replacing the Vietoris--Rips filtration
will prove beneficial in reducing the number of samples for calculating
Euclidicity~\citep{Anai20a, Sheehy13a}.
Along these lines, we also plan to derive theoretical results that
relate specific filtrations and the expressivity of the corresponding
Euclidicity measure. Another direction for future research concerns the
approximation of a manifold from inherently singular data, i.e.\ finding
the \emph{best} manifold approximation to a given data set with
singularities. This way, singularities could be resolved during
the training phase of models, provided an appropriate loss
function exists. Euclidicity may thus serve as a metric for assessing
data sets, paving the way towards more trustworthy and faithful
embeddings.

\section*{Acknowledgements} 

B.R.\ is supported by the Bavarian state government with funds from the
\emph{Hightech Agenda Bavaria}.
The authors sincerely thank Francesco Conti for stimulating
discussions.
Moreover, the authors are grateful for the stimulating discussions with
the anonymous reviewers. In particular the comments by reviewer
\texttt{By5N} helped us in refining and positioning our method better
in the context of other unsupervised machine learning methods.

\section*{Reproducibility} 
%
Our code is available under \url{https://github.com/aidos-lab/TARDIS}.
All dependencies are listed in the respective \texttt{pyproject.toml}
file, and the \texttt{README.md} discusses how to install our package
and run our experiments.

\bibliography{main}
\bibliographystyle{icml2023}

\appendix
\onecolumn

\section{Appendix}

\subsection{Notation}\label{app:Notation}

\begin{center}
  \begin{tabular}{@{}ll}
    \toprule
    Symbol & Meaning\\
    \midrule
    $\epsilon$ & local annulus scale parameter\\
    $\reals$ & real numbers \\
    $H_i$ & $i$th (ordinary) homology functor (with $\mathds{Z}/2 \mathds{Z}$ coefficients) \\
    $\tilde{H}_i$  & $i$th reduced homology functor (with $\mathds{Z}/2 \mathds{Z}$ coefficients) \\
    $\inf$ & infimum \\
    $\sup$ & supremum \\
    $|\cdot|_\infty$ & uniform~(infinity) norm \\  
    $n$ & intrinsic dimension of the space under consideration \\
    $N$ & ambient dimension of the space under consideration \\
    $\varinjlim$ & (categorical) colimit \\
    $S^k$ & $k$-dimensional sphere \\
    $c^{\circ}X:= X \times (0,1]/X \times \{ 1\}$ & open cone of
    a topological space $X$\\
    \bottomrule
  \end{tabular}
\end{center}

\subsection{Proofs of the Main Statements in the Paper}\label{app:Main proofs}

We restate the theorems from the main paper for the convenience of
readers, along with their proofs, which were removed for space
reasons.
We first prove the stability theorem, first stated on
\cpageref{thm:Stability} in the main text, which shows that our method
enjoys stability properties with respect to radius changes of the
intrinsic annuli.
\thmstability*
\begin{proof}
  The Hausdorff distance of two non-empty subsets $A,B \subset \X$ is
  $\distH(A,B) := \inf \{ \epsilon \geq 0 \mid A \subset B_\epsilon, B \subset A_\epsilon \}$,
  where $A_\epsilon = \cup_{a \in A} \{ x \in \X ;\ \dist(x,a) \leq
  \epsilon \}$ denotes the $\epsilon$-thickening of $A$ in $X$.
  Set $\epsilon:= \max \{ \epsilon_1, \epsilon_2 \}$.
  By assumption, $B_r^{s}(x) \subset
  B_{r^{\prime}}^{s^{\prime}}(x)_\epsilon$ and
  $B_{r^{\prime}}^{s^{\prime}}(x) \subset B_r^{s}(x)_\epsilon$, i.e.\
  $\distH(B_r^{s}(x),B_{r^{\prime}}^{s^{\prime}}(x)) \leq \epsilon$.
  Using the geometric stability theorem of persistence diagrams~\citep{chazal2014persistence},
  we have $\frac{1}{2} \distB(\diagram,\diagram^{\prime}) \leq \distH(B_r^{s}(x), B_{r'}^{s'}(x))$, which proves the claim.
\end{proof}

Next, we prove that our \emph{persistent intrinsic dimension}~(PID)
measure is capable of capturing the dimension of manifolds correctly,
provided sufficiently many samples are present. This theorem was first
stated on \cpageref{thm:Intrinsic dimension manifolds}.
\thmdimension*
\begin{proof}
Let $x \in \X$ be an arbitrary point. Since $M$ is a manifold, $x$ admits a Euclidean neighbourhood~$U$.
Since $M$ is smooth, we can assume $U$ to be arbitrarily close to being flat by shrinking it.
Thus, we can find $\epsilon_2>0$ with $B_r^s(x) \subset U$ for all $r,s< \epsilon_2$ such that $\Hom_i(\vietoris(B_r^s(x),t))=0$ for all $i \geq n$, and all $t$.
Hence, $\PH_i(\vietoris(B_r^s(x),\bullet))=0$ for all $i \geq n$, and therefore $i_x(\epsilon_2) \leq n$.
By contrast, for $S$ sufficiently large, and $r,s$ as before, there exists a parameter $t$ such that $\vietoris(B_r^s(x),t)$ is homotopy-equivalent to an $(n-1)$-sphere, and so $\Hom_{n-1}(\vietoris(B_r^s(x),t))$ admits a generator, i.e.\ it is non-trivial.
Consequently, $\PH_{n-1}(\vietoris(B_r^s(x),\bullet)) \neq 0$, and $i_x(\epsilon_2)=n$.
By further increasing $S$, we can ensure that the statement still holds when we decrease $\epsilon_2$, which proves the two remaining claims.
\end{proof}

Finally, we prove that \emph{Euclidicity}, an in particular its
approximation in our implementation, i.e.\ \cref{eq:Euclidicity
implementation}, is consistent when dealing with manifolds in that it
correctly assigns every point an infinitesimally small Euclidicity,
provided a sufficient number of samples is available. This theorem was
first stated on \cpageref{thm:Manifold consistency}.
\thmmanifolds*
\begin{proof}
  Since $M$ is a manifold, every point $x$ admits a Euclidean neighbourhood $U$.
  Moreover, as $M$ is smooth, we may assume~$U$ to be arbitrarily
  close to being flat~(by potentially shrinking it).
  Thus, we can find an $s_{\epsilon}$ with the desired properties such
  that
  \begin{inparaenum}[(i)]
    \item $B_0^{s_{\epsilon}}(x) \subset U \cap \X$, and
    \item $\distH(B_0^{s_{\epsilon}}(x),\EucB_0^{s_{\epsilon}}(x)) \leq \frac{\epsilon}{2}$,
  \end{inparaenum}
  where the latter condition can be achieved by increasing
  $S$ if necessary.
  Consequently, $\distH(B_r^{s}(x),\EucB_r^{s}(x))
  \leq \frac{\epsilon}{2}$ for \emph{any} $0 \leq r \leq s \leq
  s_{\epsilon}$, and by making use of the geometric stability
  theorem~\citep{chazal2014persistence}, the right-hand side of
  \cref{eq:Euclidicity implementation} reduces to
  \begin{equation*}
        \frac{1}{C} \sum_{(r, s) \in \Gamma}
  \distB\mleft(
    \PH_i\mleft(\vietoris\mleft(B_{r}^{s}\mleft(x\mright),     \bullet\mright)\mright),
    \PH_i\mleft(\vietoris\mleft(\EucB_{r}^{s}\mleft(x\mright), \bullet\mright)\mright)
  \mright) 
  \leq \frac{2}{C} \sum_{(r, s) \in \Gamma} \distH(B_r^{s}(x),\EucB_r^{s}(x)) \leq \epsilon
    \end{equation*}
    for any grid $\Gamma$ with the properties specified in the statement. This concludes the proof.
\end{proof}

\subsection{Local Homology}\label{app:Proofs}

Local homology quantifies topological properties of
infinitesimal small neighbourhoods of a fixed point.
For a topological space $X$ and $x \in X$, its $i$th local homology
group is defined as the categorical colimit $\Hom_i(X,X \setminus x) := \varinjlim_U \Hom_i(X,X
\setminus U)$, where the direct system is given by the induced maps on
homology that arise via the inclusion of~(small) neighbourhoods of
$x$. Heuristically, a local homology class can be thought of as a homology class of an infinitesimal small punctured neighbourhood of a point.
When $X$ is a simplicial complex, we may view~$x$ as a vertex in $X$,
using subdivision if necessary.
Its \emph{star}~$\St(x)$ is defined to be the union of simplices in $X$
that have $x$ as a face, whereas its \emph{link} $\Lk(x)$ consists of
all simplices in $\St(x)$ that do not have $x$ as a face.
By the excision axiom for homology, we have 
\begin{equation}
  \Hom_i(X,X \setminus x) \cong \Hom_i(\St(x),\St(x) \setminus x).
\end{equation}
Since $\St(x)$ is \emph{contractible}, the long exact reduced homology
sequence of the pair $(\St(x), \St(x) \setminus x)$ records exactness
of
\begin{align*}
0 = \tilde{\Hom}_i(\St(x)) \rightarrow \Hom_i(\St(x),\St(x) \setminus x) \rightarrow \tilde{\Hom}_{i-1}(\St(x) \setminus x) \rightarrow \tilde{\Hom}_{i-1}(\St(x)) = 0
\end{align*}
for all $i$, and therefore $\Hom_i(\St(x),\St(x) \setminus x) \cong \tilde{\Hom}_{i-1}(\St(x) \setminus x)$.
By noting that $\St(x)
\setminus x$ deformation retracts to $\Lk(x)$, and by summarising the previous observations, we obtain
\begin{equation}
  \Hom_i(X,X \setminus x) \cong \tilde{\Hom}_{i-1}(\Lk(x)).
  \label{eq:Local homology}
\end{equation} 
The \textbf{key takeaway} here is that the homology of $\Lk(x)$ will
usually differ from the homology of a sphere, once $\Lk(x)$ is not
homotopy-equivalent to a sphere.
For example, when $x$ is an isolated singularity in a stratified
simplicial complex $X$ of dimension $n$, then its distinguished
neighbourhood is given by $U \cong c^{\circ}L$.
Thus, $\Lk(x) = L$ and $\Hom_i(X,X \setminus x)
= \tilde{\Hom}_{i-1}(L)$ by \cref{eq:Local homology}, which
is usually different from $\tilde{\Hom}_{i-1}(S^{n-1})$, when $x$ does
not admit a Euclidean neighbourhood in the appropriate dimension. In particular, points in a stratified simplicial complex that lie in strata of different dimensions can be distinguished in such a way, since homology of a sphere $S^k$ records precisely one non-trivial generator in homology degree $k$.
This observation motivates and justifies using  local homology for detecting non-Euclidean neighbourhoods, and serves as the primary motivation for our
\emph{Euclidicity} measure in \cref{sec:Euclidicity}.

\subsection{Pseudocode}\label{app:Pseudocode}

We provide brief pseudocode implementations of the algorithms discussed
in \cref{sec:Methods}. In the following, we use $\#\operatorname{Bar}_i(\X)$
to denote the number of $i$-dimensional persistent barcodes of
$\X$~(w.r.t.\ the Vietoris--Rips filtration, but any other choice of
filtration affords the same description).
\cref{alg:PID} explains the calculation of \emph{persistent intrinsic
dimension}~(see \cref{sec:Persistent intrinsic dimension} in the main
paper for details). For the subsequent algorithms, we assume that the
estimated dimension of the intrinsic dimension of the data is~$n$. We
impose no additional requirements on this number; it can, in fact, be
obtained by any choice of intrinsic dimension estimation method.
As a short-hand notation, for
  $p_i=\PH_{n-1}\mleft(\vietoris\mleft(\EucB_{\bullet}^{\bullet}(x),\bullet\mright)\mright)$
w.r.t.\ some sample of $\{ y \in \reals^{n} \mid r \leq \dist(x, y) \leq s \}$,
we denote by $p_i^{r,s}=\PH_{n-1}\mleft(\vietoris\mleft(\EucB_{r}^{s}(x),\bullet\mright)\mright)$
the respective fibred persistent local homology barcode~(calculated
w.r.t.\ the same sample). \cref{alg:Euclidicity} then shows how to
calculate the \emph{Euclidicity} values, following \cref{eq:Euclidicity}
and one of its potential implementations, given in \cref{eq:Euclidicity
implementation}.

\begin{algorithm}[tbp]
  \caption{%
    An algorithm for calculating the \emph{persistent intrinsic dimension}~(PID)
  }
  \label{alg:PID}
  \begin{algorithmic}[1]
    \REQUIRE $x \in \X$, $s_{\max}$, $\ell$.
    \FOR[Iterate over the parameter grid]{$s \in \Gamma$}
    \STATE $i_x(s) \gets 0$
    \FOR{$r < s \in \Gamma$}
     \FOR{$i=1,\dots, N-1$}
       \STATE \textbf{Calculate $\#\operatorname{Bar}_i(B_r^s(x))$}
       \IF{$\#\operatorname{Bar}_i(B_r^s(x))>0$}
         \STATE $i_x(s) \gets i+1$
       \ENDIF
     \ENDFOR
    \ENDFOR
    \STATE \textbf{return} $i_x(s)$
    \ENDFOR
  \end{algorithmic}
\end{algorithm}

\begin{algorithm}[tbp]
  \caption{%
    An algorithm for calculating the \emph{Euclidicity values} $\delta_{jk}$
  }
  \label{alg:Euclidicity}
  \begin{algorithmic}[1]
    \REQUIRE $x \in \X$, $s_{\max}$, $\ell$, $n$, $\{p_1, \dotsc, p_m\}$.
    \FOR{$j=1,\dots, m$}
      \FOR{$k=j + 1,\dots, m$}
        \FOR{$s \in \Gamma$}
          \FOR{$r \in \Gamma, r<s$}
            \STATE \textbf{Calculate $\distB(p_j^{r,s},p_k^{r,s})$} \COMMENT{Calculate bottleneck distance}
            \STATE \textbf{return} $\distB(p_j^{r,s},p_k^{r,s})$
          \ENDFOR
        \ENDFOR
        \STATE \textbf{Calculate $\Dist(p_j,p_k)$} \COMMENT{Evaluate \cref{eq:Euclidicity implementation}}
        \STATE \textbf{return} $\Dist(p_j,p_k)$
      \ENDFOR
    \ENDFOR
  \end{algorithmic}
\end{algorithm}

\subsection{Stability of Euclidicity Estimates}\label{app:Stability}

\begin{figure}[tbp]
  \centering
  \includegraphics{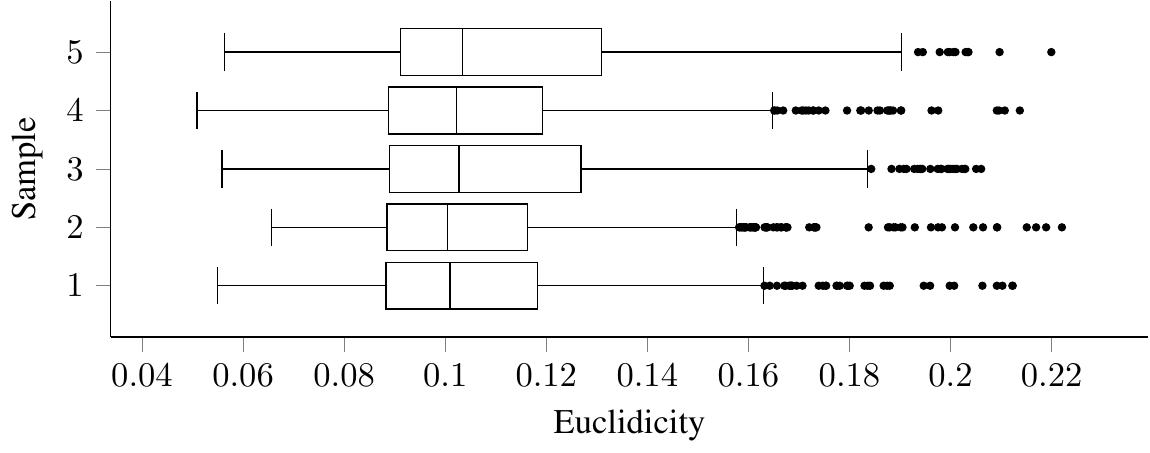}%
  \caption{%
    Boxplots of the Euclidicity values of different random samples of
    the `pinched torus' data set. While each sample invariably exhibits
    some degree of geometric variation,  we are able to
    reliably identify the singularity and its neighbourhood.
  }
  \label{fig:Pinched torus samples}
\end{figure}

\begin{figure}[tbp]
  \centering
  \subcaptionbox{$s_{\max} = 0.55$}{%
    \includegraphics {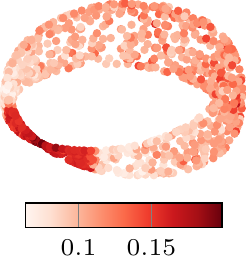}%
  }\,
  \subcaptionbox{$s_{\max} = 0.65$}{%
    \includegraphics {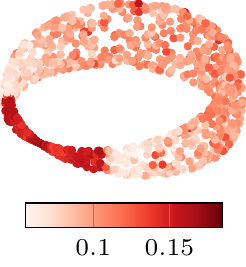}%
  }\,
  \subcaptionbox{$s_{\max} = 0.75$}{%
    \includegraphics {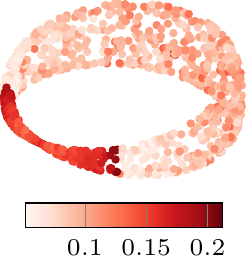}%
  }%
  \caption{%
    Modifying the outer radius $s_{\max}$ still enables us to detect the
    singularity of the `pinched torus.' Larger radii, however,
    progressively increase the field of influence of our method, thus
    starting to assign high Euclidicity values to larger regions of the
    point cloud.
  }
  \label{fig:Pinched torus outer radius point clouds}
\end{figure}

\begin{figure}[tbp]
  \centering
  \includegraphics{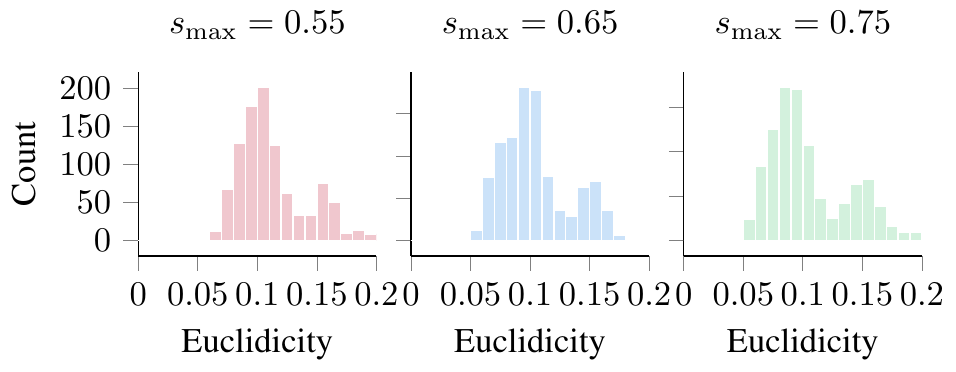}%
  \caption{%
    Histograms of the Euclidicity values for the point clouds shown in \cref{fig:Pinched torus outer radius point clouds}.
    Larger radii result in the distribution accumulating more
    probability mass at higher Euclidicity values, making the
    singularity detection procedure less local~(but still succeeding in
    detecting the singularity and its environs).
  }
  \label{fig:Pinched torus outer radius histograms}
\end{figure}

\begin{figure}[tbp]
  \centering
  \includegraphics{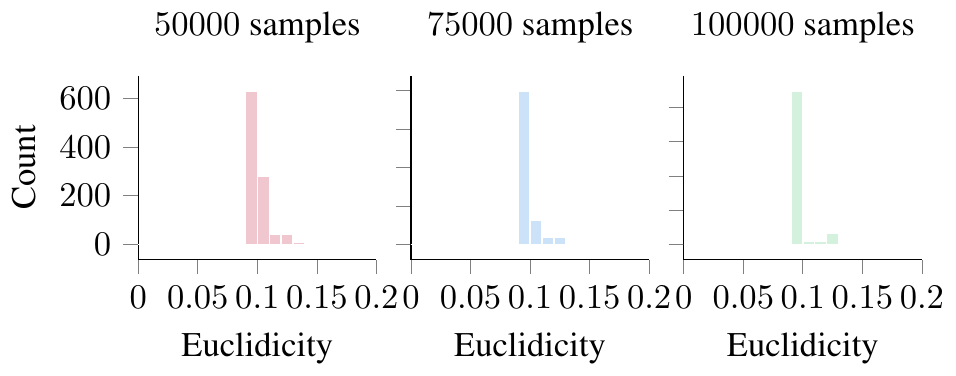}%
  \caption{%
    Histograms of the Euclidicity values for varying point sample densities. 
    Each diagram represents Euclidicity values of $1000$ random samples in 
    the respective data space.
  }
  \label{fig:Pinched torus density}
\end{figure}

\cref{fig:Pinched torus samples} shows that Euclidicity is robust under
sampling; repeating the calculations for smaller batches of the `pinched
torus' data set~($500$ points each) still lets us detect the singularity
and its neighbours reliably. This robustness is an important property in
practice where we are dealing with samples from an unknown data set
whose shape properties we want to capture. Euclidicity enables us to
perform these calculations in a robust manner.
Following the brief discussion in \cref{sec:Parameter selection}, we
show the results of varying $s_{\max}$, the outer radius of the local
annulus, for the `pinched torus' data set. \cref{fig:Pinched torus outer
radius point clouds} depicts point clouds of $1000$ samples; we observe
that the singularity, i.e.\ the `pinch point,' is always detected. For
larger radii, however, this detection becomes progressively more
\emph{global}, incorporating larger parts of the point cloud.
\cref{fig:Pinched torus outer radius histograms} depicts the
corresponding histograms; we observe the same shift in probability mass
towards the tail end of the distribution. For extremely large annuli, we
estimate that we lose a clear distinction between singular values and
non-singular values. Our data-driven parameter selection procedure is
thus to be preferred in practice since it incorporates data density.
However, even for markedly varying densities, we find that Euclidicity leads
to reasonably stable results, as can be seen in \cref{fig:Pinched torus density}.
Here, we calculated Euclidicity scores for the pinched torus for varying 
data sample densities of $50000,75000$ and $100000$ samples. Each diagram 
shows the histogram of Euclidicity values for $1000$ random samples in the respective 
data space. Although the histograms differ, the distributions exhibit
the same overall characteristics.

\subsection{Comparison of PID With Other Dimension Estimates}\label{app:PID}

In order to assess the quality of PID, we decided to test its
performance on a space that is both singular and has non-constant
dimension. The data space we chose consists of 2000 samples of $S^1 \vee
S^2$, i.e.\ a $1$-sphere glued together with a $2$-sphere at a certain
concatenation point. We then applied the PID procedure for a maximum
locality scale that was given by the $k$ nearest neighbour distances, for
$k \in \{25, 50, 75, 100, 125, 150, 175, 200\}$. We assigned to
each point the average of the PID scores at the respective scales that
are less than or equal to the $k$ nearest neighbour bound.
Subsequently, we compared the results with other local dimension
estimates for the respective number of neighbours. The methods that were
chosen for comparison include \texttt{lpca}, \texttt{twoNN},
\texttt{KNN}, and \texttt{DANCo}; we used the respective implementation
from the \texttt{scikit-dimension} Python
package.\footnote{\url{https://scikit-dimension.readthedocs.io/en/latest/}}

\begin{figure}[tbp]
  \centering
  \subcaptionbox{PID\label{fig:PID_S1vS2}}{%
    \includegraphics[width=0.50\textwidth]{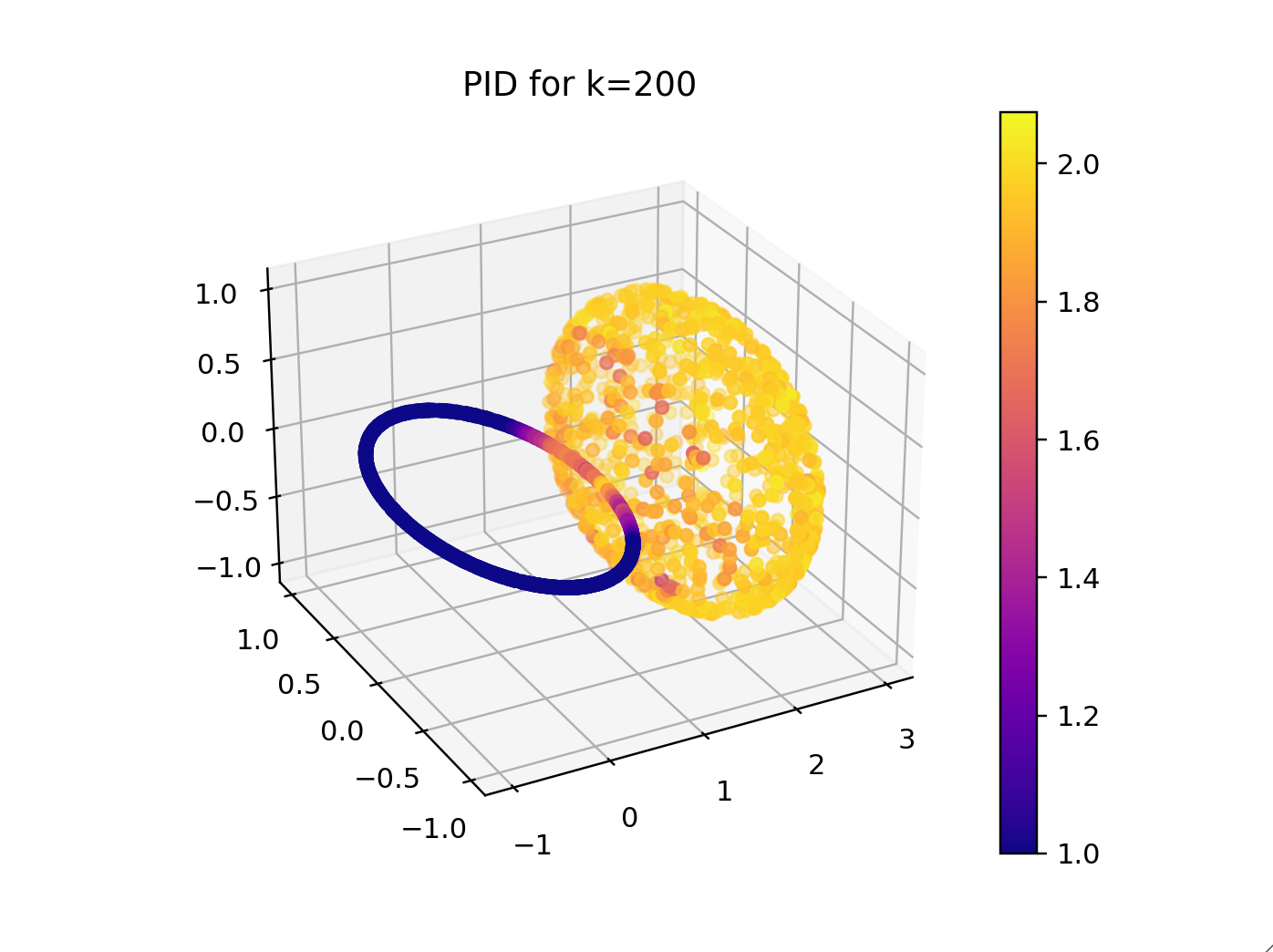}%
  }%
  \subcaptionbox{twoNN\label{fig:twoNN_S1vS2}}{%
    \includegraphics[width=0.50\textwidth]{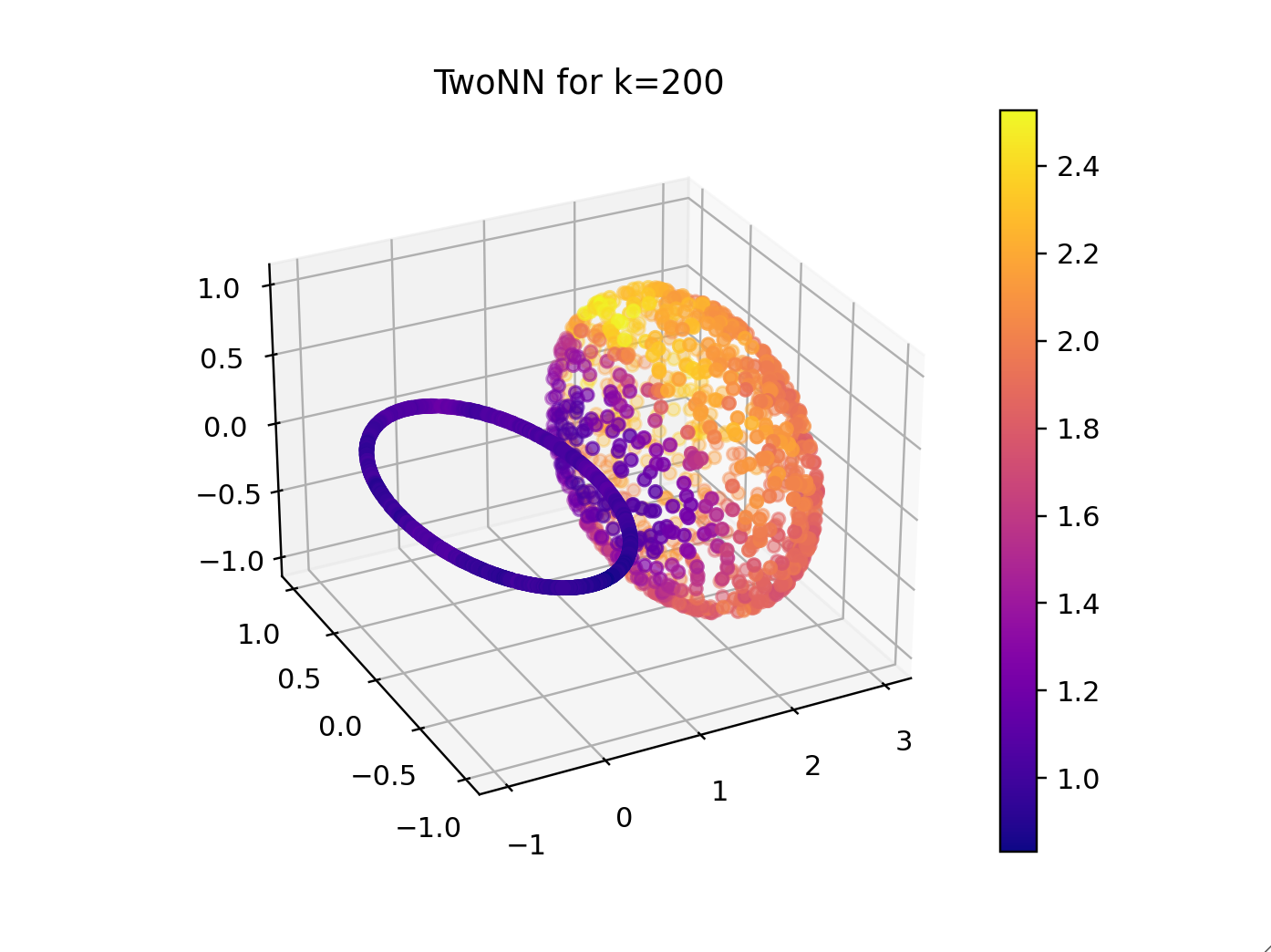}%
  }
  \caption{
    Even for large values of~$k$, PID still does not overestimate the
    local dimensionality of the data, exhibiting a clear distinction
    between the circle and the sphere, respectively.
  }
  \label{fig:app:PID vs twoNN}
\end{figure}

\cref{fig:PID_S1vS2} shows the PID results for a maximum locality scale
of $200$ neighbours, with colours showing the estimated dimension
values for each point. Overall, the correct intrinsic dimension is
detected for most of the points. However, points that lie close to the
singular point show a PID value between~$1$ and $2$. Similarly to what we
already discussed for Euclidicity, PID should therefore also be
interpreted as a measure that incorporates the intrinsic dimension of
a point on \emph{multiple scales} of locality. For real-world data, the
dimension will generally change when changing the locality scale.
However, since there is no canonical choice of scale, we believe that
any such scale provides valuable information about the intrinsic
dimension that is worth being measured. We therefore argue that
a multi-scale approach like ours is appropriate in practice, especially
in a regime that is agnostic with respect to the underlying intrinsic
dimension.
By contrast, \cref{fig:twoNN_S1vS2} shows the corresponding dimension
estimates for twoNN, where we observe less stable and reliable results
across the dataset.

\begin{figure}[tbp]
  \centering
  \subcaptionbox{$1$-sphere dimension estimates\label{fig:dim_estimates_1D}}{%
    \includegraphics[width=0.5\textwidth]{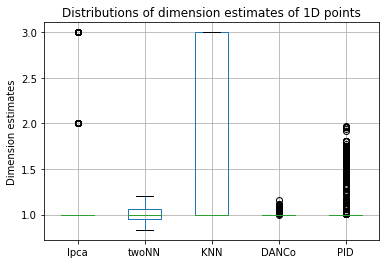}%
  }%
  \subcaptionbox{$2$-sphere dimension estimates\label{fig:dim_estimates_2D}}{%
    \includegraphics[width=0.5\textwidth]{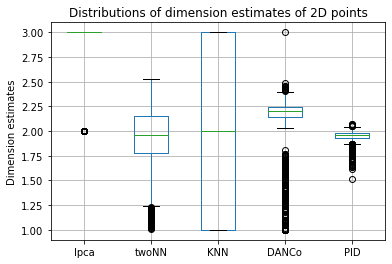}%
  }
  \caption{%
    Estimates of the local intrinsic dimension for points that are close
    to the 1D-sphere, i.e.\ the circle, or the 2D-sphere, respectively.
  }
\end{figure}

\cref{fig:dim_estimates_1D} shows boxplots of the distributions of the
dimension estimates, for all points that lie on the 1D-sphere. We see
that for PID, the mass is concentrated at a value of $1$. Although there
are outliers present, these correspond to points that are close to the
singularity, as it was expected. We note that other methods like
\texttt{KNN} and \texttt{lpca} might highly overestimate the dimension,
and that the interquartile range is significantly higher for
\texttt{twoNN} and \texttt{KNN}.
\cref{fig:dim_estimates_2D} shows the same distributions for the points
that lie on the 2D-sphere. Again, \texttt{lpca} highly overestimates the
dimension since the median lies at a value of $3$. Again, the
interquartile range of PID is the tightest, and the estimates are
closest to the ground truth. Moreover, the lower-value outliers again correspond
to points that are close to the singular gluing point.

\begin{figure}[tbp]
  \centering
  \subcaptionbox{\label{fig:dimension_estimates_diff_k_1D}}{%
    \includegraphics[width=0.5\textwidth]{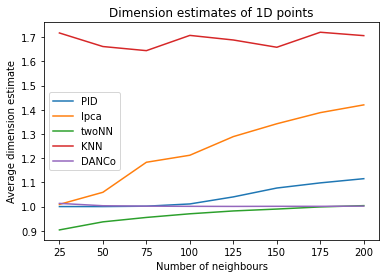}%
  }%
  \subcaptionbox{\label{fig:dimension_estimates_diff_k_2D}}{%
    \includegraphics[width=0.5\textwidth]{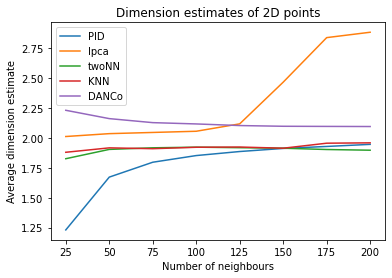}%
  }
  \caption{
    Dimension estimates of the 1D-sphere and the 2D-sphere for different
    methods, plotted as a function of the number of neighours~$k$.
  }
\end{figure}

\cref{fig:dimension_estimates_diff_k_1D} and
\cref{fig:dimension_estimates_diff_k_2D} show average dimension estimate
scores of all investigated methods for varying values of~$k$, both for
points on the $1$-sphere and the $2$-sphere.
We note that on average, only \texttt{twoNN} and \texttt{DANCo} lead to
results which are comparable with the reliability of our method.
However, as we already saw in \cref{fig:dim_estimates_1D} and
\cref{fig:dim_estimates_2D}, the variance of the scores of our method is
significantly lower, leading to more reliable outputs for each of the
points.

\subsection{Euclidicity is More Expressive than Single-Parameter Approaches}\label{sec:Euclidicity versus single-parameter}

\begin{figure}[tbp]
  \centering
  \subcaptionbox{One-parameter approach\label{app:sfig:Wedged spheres Geometric Anomaly Score}}{%
    \includegraphics[width=0.20\linewidth]{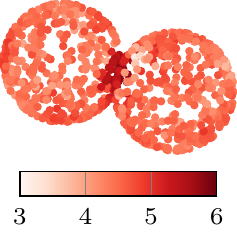}%
  }
  \quad
  \subcaptionbox{\emph{Euclidicity}\label{app:sfig:Wedged spheres Euclidicity}}{%
    \includegraphics[width=0.20\linewidth]{Figures/Wedged_spheres_Euclidicity}%
  }
  \caption{%
    In 2D, \emph{Euclidicity}~\subref{app:sfig:Wedged spheres
    Euclidicity} results in a clearly-delineated singular region~(also
    shown in the main paper), when
    compared to a single-parameter score~\subref{app:sfig:Wedged spheres
    Geometric Anomaly Score}.
  }
  \label{fig:Euclidicity versus single-parameter}
\end{figure}

\begin{figure}[tbp]
  \centering
  \subcaptionbox{}{%
    \includegraphics[height=4cm]{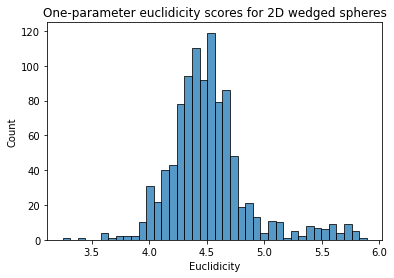}
  }%
  \subcaptionbox{}{%
    \includegraphics[height=4cm]{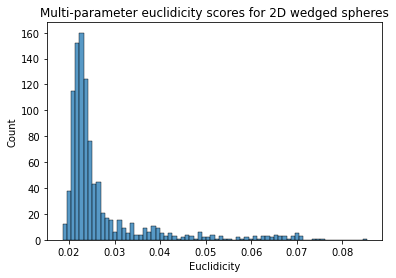}
  }
  \caption{%
    A comparison of Euclidicity values of a one-parameter
    approach~(left) and our multi-parameter approach~(right)
    demonstrates that multiple scales are necessary to adequately
    capture singularities.
  }
  \label{fig:app:One-parameter versus multi-parameter}
\end{figure}

\cref{fig:Euclidicity versus single-parameter} depicts a comparison
between Euclidicity, an inherently multi-parametric approach, and
a single-parameter approach, which assumes the \emph{same} locality
properties to hold for all points. This method is similar to the one
described by \citet{Stolz20a}. We observe that both methods are capable
of detecting the isolated singularity of $S^2 \vee S^2$, with
Euclidicity leading to a score that more clearly delineates the region
around the singularity. Notice that this behaviour will become even
more crucial in real-world data, which may exhibit strong sparsity
patterns and varying density.

Further underscoring these findings, \cref{fig:app:One-parameter versus
multi-parameter} shows the empirical distributions of Euclidicity scores
for fixed locality parameters~(left) and for our proposed multi-scale
locality approach~(right). We see that the variance is
\emph{significantly lower} in the multi-scale regime, indicating more
stable and robust results. Moreover, the ratio of maximum and mean is
higher in the multi-parameter setting, where high Euclidicity scores
correspond to data points that lie close to the singularity, resulting
in more reliable outcomes.

\subsection{Euclidicity of \data{MNIST} and \data{FashionMNIST}}\label{app:emp_dist}

\cref{fig:app:MNIST} and \cref{fig:app:FashionMNIST} show the Euclidicity
results for the 4 additional runs on both the \data{MNIST} and
\data{FashionMNIST} data sets.
Again, we depicted the 9 images with
lowest~(left), medium~(middle), and highest~(right) Euclidicity scores
for the two datasets. Moving from left to right, the images exhibit
increases in the complexity of the local geometry, giving evidence for
the reproducibility of the observation we remarked in
\cref{subsec:MNISTandFashionMNIST}.

\begin{figure}
  \centering
  \subcaptionbox{}{%
    \begin{tabular}{@{}l}
      \includegraphics[width=0.15\textwidth]{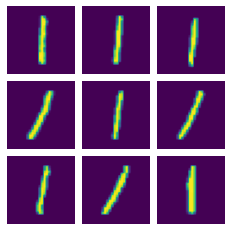}\\ 
      \includegraphics[width=0.15\textwidth]{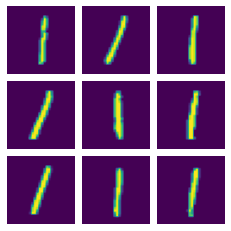}\\ 
      \includegraphics[width=0.15\textwidth]{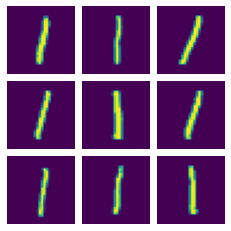}\\
      \includegraphics[width=0.15\textwidth]{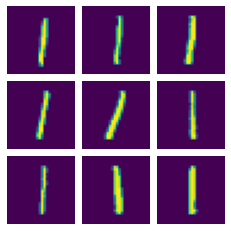}
    \end{tabular}
  }
  \subcaptionbox{}{%
    \begin{tabular}{@{}l}
      \includegraphics[width=0.15\textwidth]{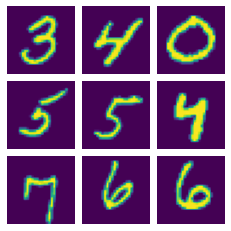}\\
      \includegraphics[width=0.15\textwidth]{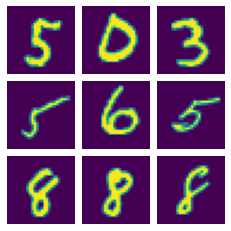}\\
      \includegraphics[width=0.15\textwidth]{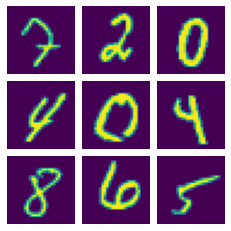}\\
      \includegraphics[width=0.15\textwidth]{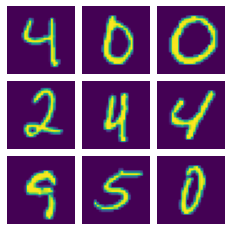}
    \end{tabular}
  }
  \subcaptionbox{}{%
    \begin{tabular}{@{}l}
      \includegraphics[width=0.15\textwidth]{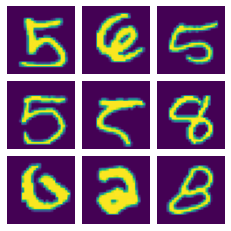}\\
      \includegraphics[width=0.15\textwidth]{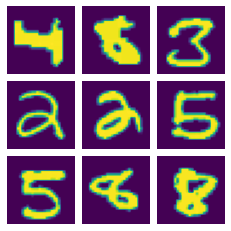}\\
      \includegraphics[width=0.15\textwidth]{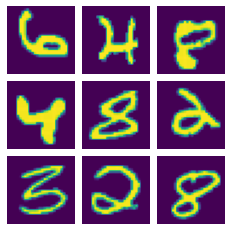}\\
      \includegraphics[width=0.15\textwidth]{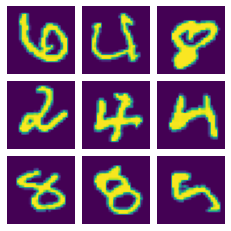}
    \end{tabular}
  }
  \caption{%
    From left to right: more examples of low Euclidicity values, median
    Euclidicity values, and high Euclidicity values for the \data{MNIST}
    data set.
  }
  \label{fig:app:MNIST}
\end{figure}

\begin{figure}
  \centering
  \subcaptionbox{}{%
    \begin{tabular}{@{}l}
      \includegraphics[width=0.15\textwidth]{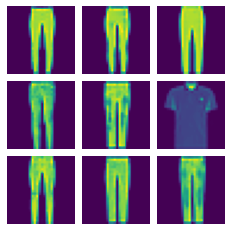}\\ 
      \includegraphics[width=0.15\textwidth]{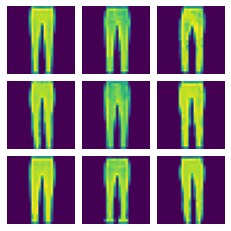}\\ 
      \includegraphics[width=0.15\textwidth]{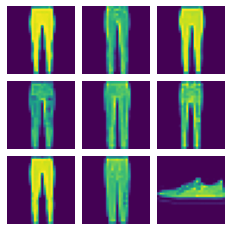}\\
      \includegraphics[width=0.15\textwidth]{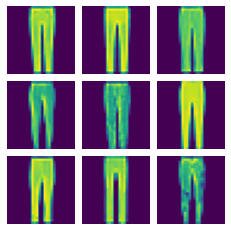}
    \end{tabular}
  }
  \subcaptionbox{}{%
    \begin{tabular}{@{}l}
      \includegraphics[width=0.15\textwidth]{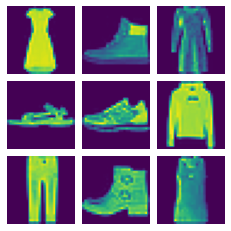}\\
      \includegraphics[width=0.15\textwidth]{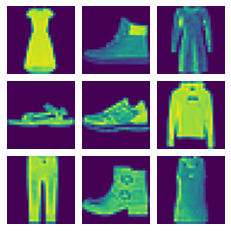}\\
      \includegraphics[width=0.15\textwidth]{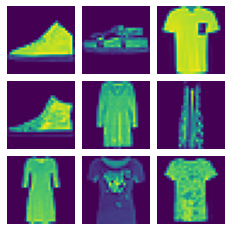}\\
      \includegraphics[width=0.15\textwidth]{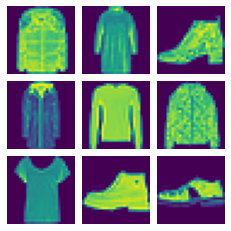}
    \end{tabular}
  }
  \subcaptionbox{}{%
    \begin{tabular}{@{}l}
      \includegraphics[width=0.15\textwidth]{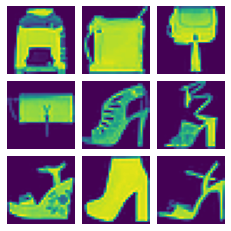}\\
      \includegraphics[width=0.15\textwidth]{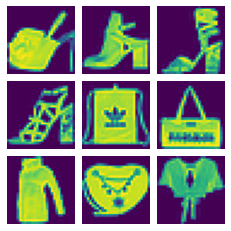}\\
      \includegraphics[width=0.15\textwidth]{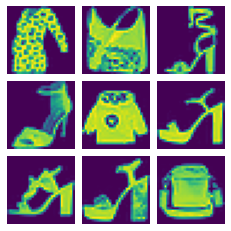}\\
      \includegraphics[width=0.15\textwidth]{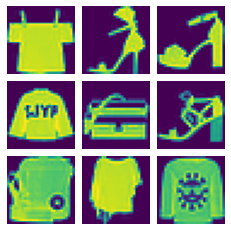}
    \end{tabular}
  }
  \caption{%
    From left to right: more examples of low Euclidicity values, median
    Euclidicity values, and high Euclidicity values for the \data{FashionMNIST}
    data set.
  }
  \label{fig:app:FashionMNIST}
\end{figure}

\clearpage

Moreover, as \cref{fig:app:MNIST versus FashionMNIST} shows, the
empirical distributions of the calculated Euclidicity scores differ
significantly for the \data{MNIST} and \data{FashionMNIST} data sets,
with the distribution for \data{MNIST} exhibiting a bimodal behaviour,
whereas the \data{FashionMNIST} Euclidicity value distribution is
unimodal. We hypothesise that this corresponds to regions of simple
complexity---and locally linear structures---in the \data{MNIST} data
set, which are absent in the \data{FashionMNIST} data set. 
Finally, we observe that Euclidicity values for \data{MNIST} are 
surprisingly stable with respect to specific choices of the intrinsic dimension,
as well as the nearest neighbour parameter $k$, which can be seen in
\cref{fig:MNIST_different_dimensions,fig:MNIST_different_k}. Here, we 
calculated Euclidicity scores for \data{MNIST} for intrinsic dimensions
of $5,10$ and $15$, respectively, and $k \in \{25,50,75\}$. Although the
scale of locality for which Euclidicity is computed varies significantly 
for different $k$, the respective distributions of Euclidicity values
appear to be almost indistinguishable. We surmise that this is
a consequence of the multi-scale nature of our approach.

\begin{figure}[tbp]
  \centering
  \subcaptionbox{MNIST Euclidicity scores for different choices of the intrinsic dimension.\label{sfig:MNIST_different_dimensions}}{
    \includegraphics[width=0.300\linewidth]{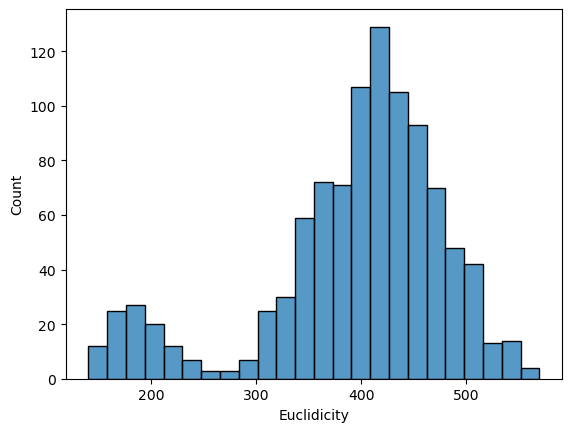}%
    \includegraphics[width=0.300\linewidth]{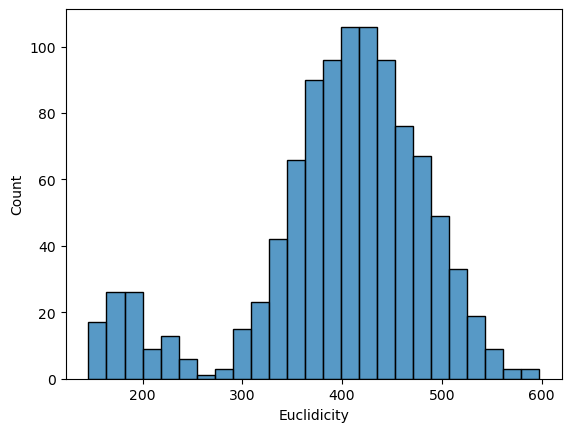}%
    \includegraphics[width=0.300\linewidth]{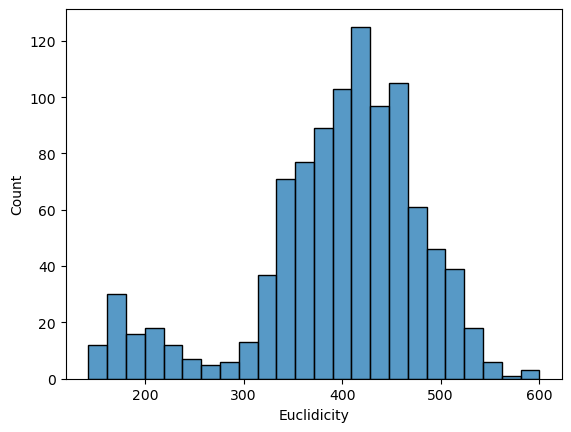}
  }
  \caption{%
    Left to right: distribution of MNIST Euclidicity scores for 
    intrinsic dimension of $5$, $10$, and $15$, respectively. As
    a consequence of the stability of our method, the distributions
    remain similar.
  }
  \label{fig:MNIST_different_dimensions}
\end{figure}

\begin{figure}[tbp]
  \centering
  \subcaptionbox{MNIST Euclidicity scores for different choices of the nearest neighbour parameter $k$.\label{sfig:MNIST_different_k}}{
    \includegraphics[width=0.300\linewidth]{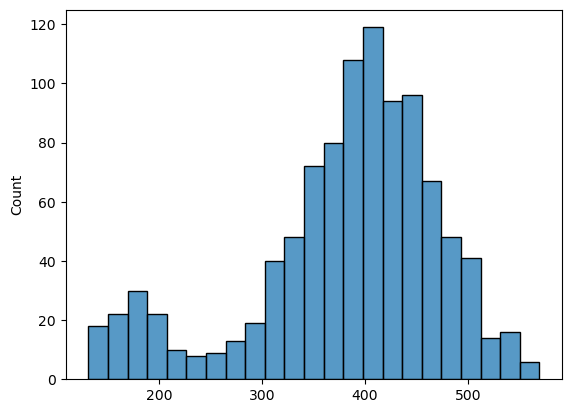}%
    \includegraphics[width=0.300\linewidth]{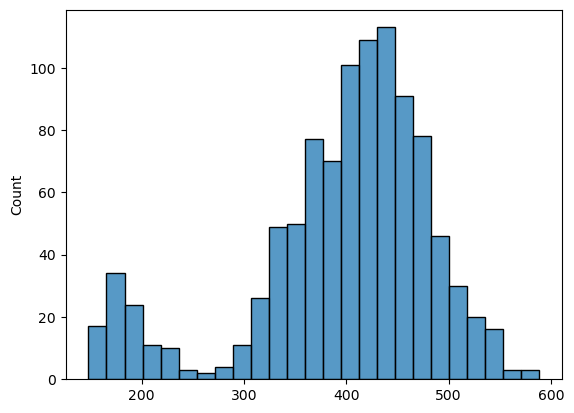}%
    \includegraphics[width=0.300\linewidth]{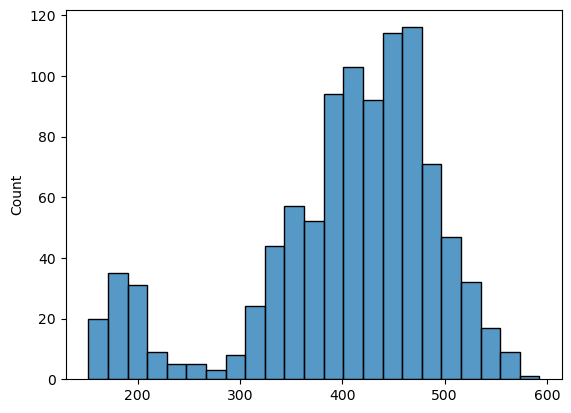}
  }
  \caption{%
    Left to right: distribution of MNIST Euclidicity scores for $k \in
    \{25,50,75\}$, respectively. As a consequence of the stability of
    our method, the distributions remain similar.
  }
  \label{fig:MNIST_different_k}
\end{figure}

\paragraph{Euclidicity and anomaly scores.}
\cref{fig:app:umap} shows a comparison of Euclidicity scores and the
scores of \texttt{IsolationForest}, an anomaly detection algorithm,
for a UMAP projection of \data{MNIST}. Here, we
used the implementation for \texttt{IsolationForest} from the
\texttt{scikit-learn} Python
package.\footnote{\url{https://scikit-learn.org/stable/}}
The clusters correspond to the different classes of \data{MNIST}, see
\cref{fig:app:umap_classes}. Although the original purpose of
Euclidicity is the detection of non-Euclidean neighbourhoods of points,
we observe a correlation of low Euclidicity scores~(i.e.\ points that are
closer to admitting a Euclidean neighbourhood) and high
\texttt{IsolationForest} scores (i.e.\ points that are less likely to be
considered anomalies). This provides further evidence that Euclidicity
could be expected to serve as a powerful anomaly detection method, on
top of its ability to capture local geometric complexity. 
However, in our experiments we observe a~(linear) correlation of only
$0.5$--$0.7$, indicating that Euclidicity detects
\emph{different} anomalous behaviour than just plain outliers.
We plan on investigating this in future work.

\begin{figure}[tbp]
  \centering
  \subcaptionbox{UMAP with class labels\label{fig:app:umap_classes}}{%
    \includegraphics[width=0.33\textwidth]{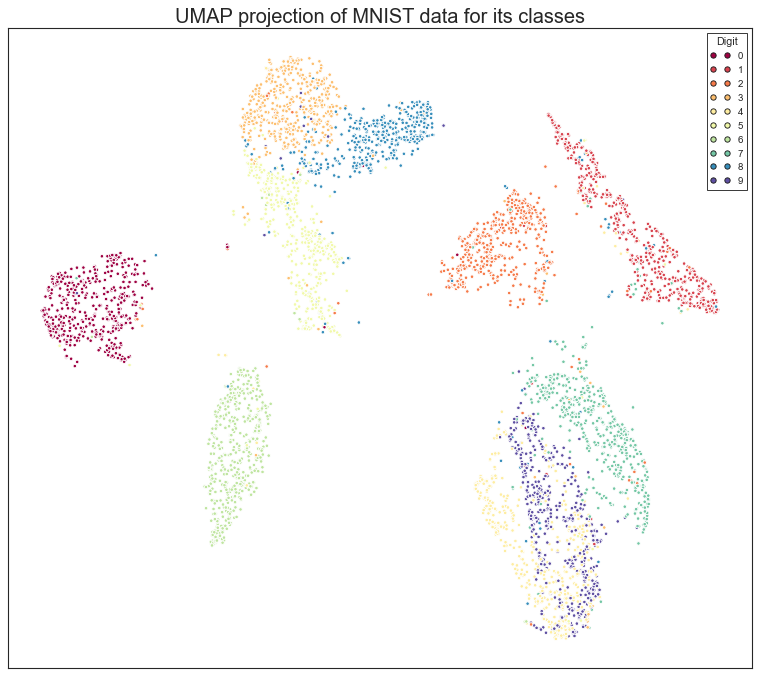}%
  }%
  \subcaptionbox{UMAP Euclidicity scores\label{fig:umap_eucl}}{%
    \includegraphics[width=0.33\textwidth]{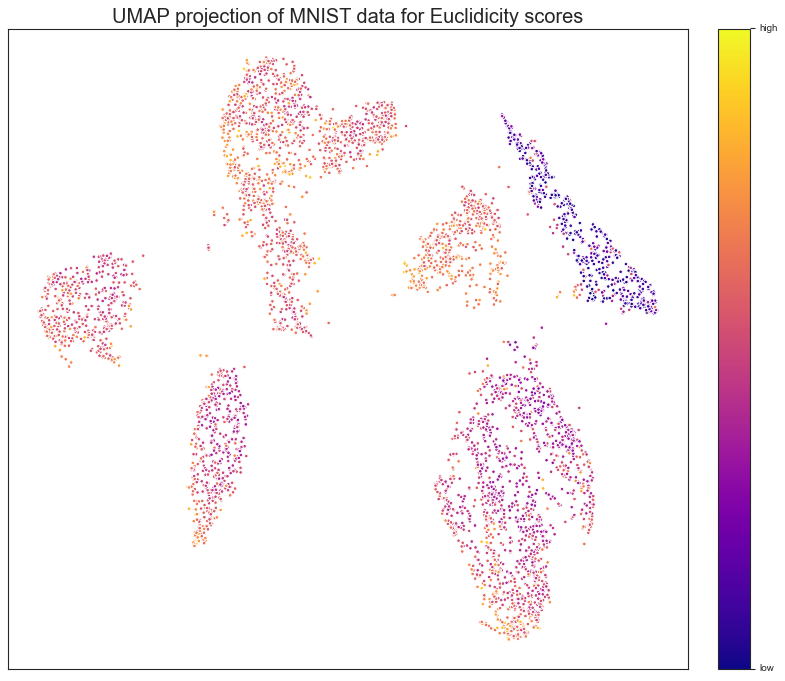}%
  }%
  \subcaptionbox{UMAP Isolation Forest scores\label{fig:umap_iso}}{%
    \includegraphics[width=0.33\textwidth]{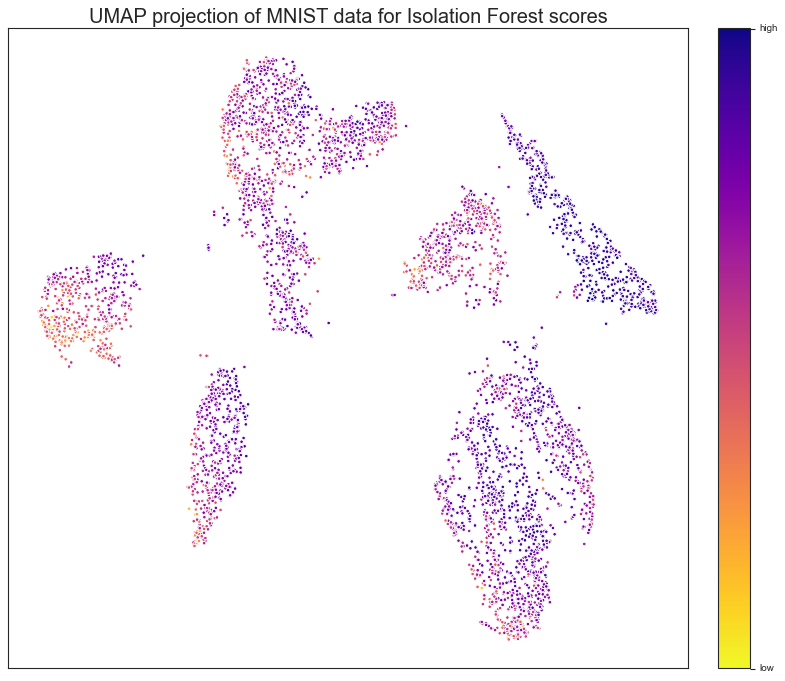}%
  }%
  \caption{
    Euclidicity correlates with Isolation Forest scores for \data{MNIST}
    samples. Notice that we switched the colour maps because \emph{low}
    Isolation Forest scores indicate anomalies.
  }
  \label{fig:app:umap}
\end{figure}

\paragraph{Details on neural network misclassification analysis.}
%
The neural network that we trained for the analysis in
\cref{subsec:MNISTandFashionMNIST} consists of an input layer of $784$
nodes, one dense hidden layer of $5$ nodes and an output layer possessing
$10$ nodes, respectively. The same architecture was used for both
\data{MNIST} and \data{FashionMNIST}, resulting in $3985$ trainable
parameters. Accuracy scores are $0.777$ for \data{MNIST}, and $0.714$ for
\data{FashionMNIST}.
\ul{Neither these scores nor the accuracy itself are
supposed to be competitive; we merely use this network as an example
to outline how Euclidicity could be potentially used to inform network
training and highlight potentially problematic regions in data}.
\cref{fig:Euclidicity misclassified samples} shows the respective results of
Euclidicity scores for misclassified versus correctly-classified samples.
To compare the statistical significance of the respective means for
correctly and misclassified samples, we performed Welch’s t-test, as the
samples are approximately normally distributed. We obtained p-values of
\num{2.8E-5} in the case of \data{MNIST}, and \num{5.7E-3} in the
case of \data{FashionMNIST}. With respect to the common rejection
thresholds in the literature, this leads to a clear
rejection of the null hypothesis, which assumes the population means to
be equal. However, we only claim a \emph{correlation} between the likelihood
for samples to be correctly classified and their corresponding
Euclidicity scores. A statement in terms of \emph{causality} is to be
investigated in future work.

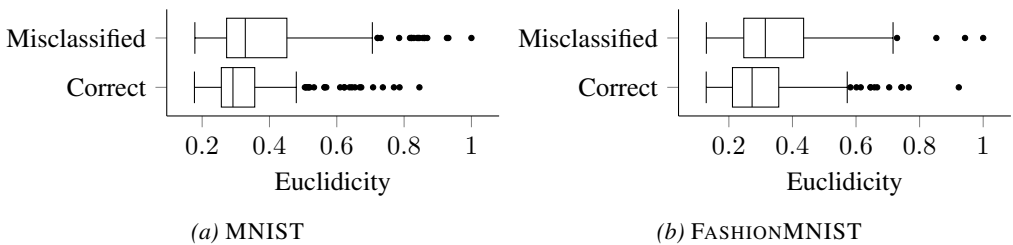
\begin{figure}[tbp]
  \centering
  \pgfplotsset{%
    boxplot/draw direction = x,
  }%
  \subcaptionbox{\data{MNIST}}{%
    \begin{tikzpicture}
      \begin{axis}[%
        axis x line*     = bottom,
        axis y line*     = none,
        ytick            = {1, 2},
        yticklabels      = {Correct, Misclassified},
        height           = 3.0cm,
        width            = 6.0cm,
        mark size        = 1pt,
        tick align       = outside,
        xlabel           = {Euclidicity},
        xtick            = {0.2, 0.4, 0.6, 0.8, 1.0},
        ylabel           = {},
        xticklabel style = {%
          /pgf/number format/fixed,
          /pgf/number format/precision=2,
        },%
      ]%
        \addplot[boxplot] table[y index = 0] {Data/mnist_correct_euclidicity.txt};
        \addplot[boxplot] table[y index = 0] {Data/mnist_errors_euclidicity.txt};
      \end{axis}%
    \end{tikzpicture}
  }
  \subcaptionbox{\data{FashionMNIST}}{%
    \begin{tikzpicture}
      \begin{axis}[%
        axis x line*     = bottom,
        axis y line*     = none,
        ytick            = {1, 2},
        yticklabels      = {Correct, Misclassified},
        height           = 3.0cm,
        width            = 6.0cm,
        mark size        = 1pt,
        tick align       = outside,
        xlabel           = {Euclidicity},
        ylabel           = {},
        xtick            = {0.2, 0.4, 0.6, 0.8, 1.0},
        xticklabel style = {%
          /pgf/number format/fixed,
          /pgf/number format/precision=2,
        },
      ]
        \addplot[boxplot] table[y index = 0] {Data/fmnist_correct_euclidicity.txt};
        \addplot[boxplot] table[y index = 0] {Data/fmnist_errors_euclidicity.txt};
      \end{axis}
    \end{tikzpicture}%
  }%
  \caption{%
    A comparison of Euclidicity scores for misclassified and correctly
    classified samples in two image data sets.
  }%
  \label{fig:Euclidicity misclassified samples}
\end{figure}

\begin{figure}
  \centering
  \subcaptionbox{\data{MNIST}}{%
    \includegraphics[height=4cm]{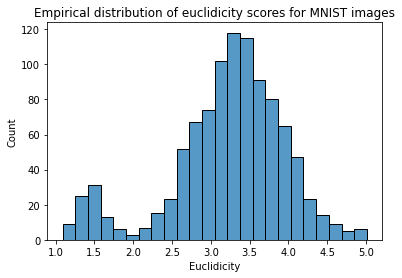}%
  }%
  \subcaptionbox{\data{FashionMNIST}}{%
    \includegraphics[height=4cm]{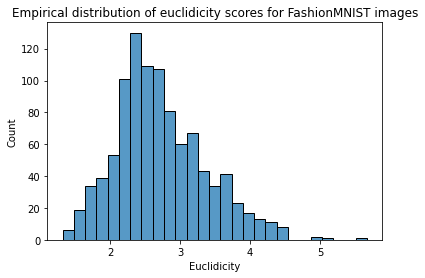}%
  }%
  \caption{%
    Both \data{MNIST} and \data{FashionMNIST} exhibit markedly different
    distributions in terms of Euclidicity: \data{MNIST} Euclidicity
    values are bimodal, whereas \data{FashionMNIST} Euclidicity values
    are unimodal.
  }
  \label{fig:app:MNIST versus FashionMNIST}
\end{figure}

\subsection{Euclidicity Captures Geometric Complexity in Histology Data}

Finally, we applied Euclidicity to a benchmark histology data
set.\footnote{\url{https://github.com/basveeling/pcam}}
The $96 \times 96$ pixel images were first scaled down to $60 \times 60$
pixels, to speed up the computation. Moreover, since the images
are inherently coloured, we transformed them to greyscale images in
order to naturally obtain a high-dimensional representation in Euclidean
space, by flattening them. We then randomly picked $1000$ of the training
images and calculated Euclidicity scores of these samples~(with respect
to the whole space of training images).
The intrinsic dimension was chosen to be $35$, as suggested by
\texttt{twoNN}. All the other hyperparameters were chosen as in the
setting of \cref{app:emp_dist}. The results of images admitting different
scores are depicted in \cref{fig:histology data sets examples}.
Similarly to the observations in \cref{app:emp_dist}, we see very
different structural behaviour inside of the images for different
Euclidicity values: while low scores correspond to images that admit
large voids, medium-score images correspond to images with finer local
geometric structure. Images of high Euclidicity scores, however, seem to
admit \emph{both} large voids \emph{and} regions of fine-grained geometric structure.
We therefore hypothesise the potential utility of Euclidicity for
biological data, for instance to ensure data quality of~(histology or
biomedical) image datasets.

\begin{figure}[tbp]
  \centering
  \includegraphics[width=0.200\linewidth]{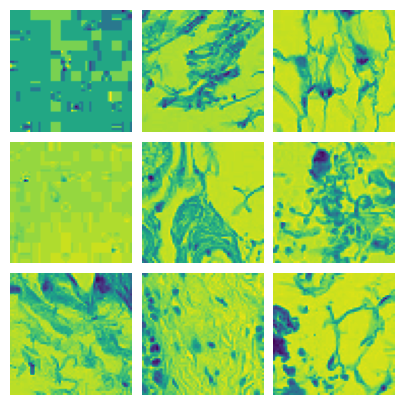}%
  \includegraphics[width=0.200\linewidth]{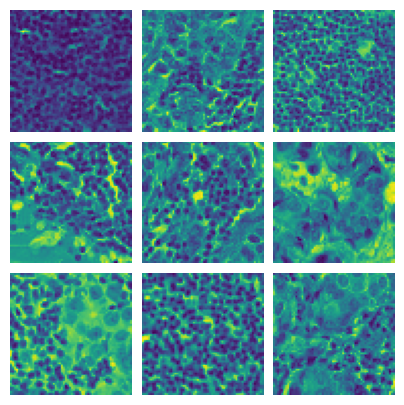}%
  \includegraphics[width=0.200\linewidth]{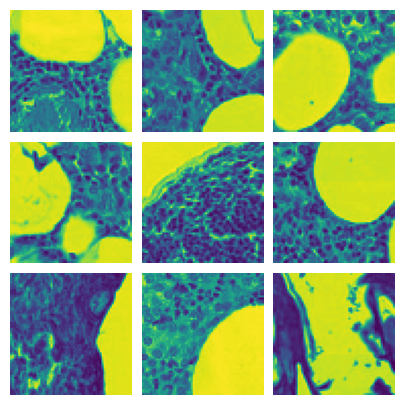}
  \caption{%
    Left to right: sample histology images exhibiting low, median, and
    high Euclidicity, respectively. All samples are taken from the
    `PatchCamelyon' benchmark data set.
  }
  \label{fig:histology data sets examples}
\end{figure}

\end{document}